\documentclass{article}

\usepackage{iclr2027_conference,times}

\usepackage{my/myconfig}

\usepackage[utf8]{inputenc}
\usepackage[T1]{fontenc}
\usepackage{hyperref}
\usepackage{url}
\usepackage{booktabs}
\usepackage{amsfonts}
\usepackage{nicefrac}
\usepackage{microtype}
\usepackage{soul}
\usepackage{tcolorbox}
\usepackage{pifont}
\usepackage{xspace}

\newcommand{\system}{\textsc{EquiMem}\xspace}

\title{
\system: Calibrating Shared Memory in Multi-Agent Debate
via Game-Theoretic Equilibrium
}

\author{
\textbf{Yuqiao Meng}$^{1}$ \quad
\textbf{Luoxi Tang}$^{1}$ \quad
\textbf{Sakshi Sunil Narvekar}$^{1}$ \quad
\textbf{Rupali Rajendra Vaje}$^{1}$ \quad \\
\textbf{Yingxue Zhang}$^{1}$ \quad
\textbf{Muchao Ye}$^{2}$ \quad
\textbf{Zhaohan Xi}$^{1}$ \\
\\
$^{1}$Binghamton University, State University of New York,
Binghamton, NY, USA \\
$^{2}$University of Iowa, Iowa City, IA, USA \\
\texttt{zxi@binghamton.edu}
}

\iclrfinalcopy
\begin{document}

\maketitle

\pagestyle{fancy}
\fancyhf{}
\fancyhead[L]{Preprint.}
\renewcommand{\headrulewidth}{0.4pt}


\begin{abstract}

Multi-agent debate (MAD) systems increasingly rely on shared memory to support long-horizon reasoning, but this convenience opens a critical vulnerability: a single corrupted entry can contaminate the downstream memory-augmented reasoning, and debate alone fails to filter such errors. Existing safeguards filter entries via heuristics or LLM-based validation, yet they rely on AI judgments that share the same failure modes and overlook the cross-agent dynamics of MAD. We address this gap by formulating memory updating in MAD as a \textbf{zero-trust memory game}, in which no agent is assumed reliable and the game's equilibrium motivates a principled objective for calibrating memory influence. Guided by this objective, we propose \system, an inference-time calibration mechanism that evaluates each update against the shared memory state and the memory already being used by other agents, using their existing retrieval queries and traversal paths as evidence  of calibration. \system instantiates this calibration for both embedding- and graph-based memory, and across diverse benchmarks, MAD frameworks, and memory architectures, it consistently outperforms existing safeguards, remains robust under adversarial agents, and incurs negligible inference overhead.
\end{abstract}
\section{Introduction}
\label{sec:intro}

Multi-agent debate (MAD) systems built on large language models (LLMs) have shown strong performance on complex reasoning~\citep{du2024improving,li2024improving,wang2025learning,tang2025test,more2026theramind}, embodied action~\citep{shridhar2020alfworld,wang2025madra}, and planning~\citep{hu2025debate,li2025swe,ma2024agentboard} tasks, where agents iteratively discuss, critique, and refine each other's outputs~\citep{meng2026small,chan2023chateval,liang2024encouraging}. To support interactions beyond a single round, recent MAD systems add a \emph{shared memory} that persists intermediate reasoning, past actions, and episodic trajectories across rounds~\citep{anokhin2024arigraph,wang2023voyager,zhang2025g,zhong2024memorybank}.

While shared memory supports long-horizon reasoning, it also creates a new source of risk: once an unreliable contribution is written into memory, it can affect later retrieval and reasoning long after the original debate round. This risk is difficult to avoid because the agents themselves are imperfect. LLMs may hallucinate, agree sycophantically~\citep{banks2026ai,carro2024flattering,sharma2023towards}, or state incorrect claims with high confidence~\citep{shinn2023reflexion,zhang2025g}. Debate does not necessarily correct these errors. As illustrated in Figure~\ref{fig:intro}, a confident but incorrect contributor may face little resistance from uncertain peers~\citep{du2024improving,wang2024rethinking,xiong2023examining}; a correct but tentative contribution may be rejected by more assertive auditors~\citep{cui2024or,zhang2025falsereject}; or widespread uncertainty may leave useful information out of memory altogether~\citep{song2025llms,tomani2024uncertainty,tang2026value}. These examples illustrate a common difficulty: the behavior of debating agents is itself an unreliable basis for deciding how much a new memory update should be trusted.

Existing safeguards reduce this risk in different ways. Some score proposed entries with heuristics or auxiliary LLMs~\citep{alon2023detecting,wang2024self}, while others validate updates through agreement across related reasoning paths~\citep{wei2025memguard}. These methods can filter suspicious content, but they still leave two limitations. First, methods that rely on additional LLM judgments inherit the same model errors that make memory updates unreliable in the first place. Second, most safeguards evaluate an update mainly from the update itself or from local agreement, rather than asking how it fits into the shared memory that has already accumulated and how it would affect information that other agents are currently using. This leads to our central question: \emph{how can we calibrate memory updates in multi-agent debate without relying on any individual agent to judge whether an update is trustworthy?}


\begin{figure}[t]
    \centering
    \includegraphics[width=0.82\linewidth]{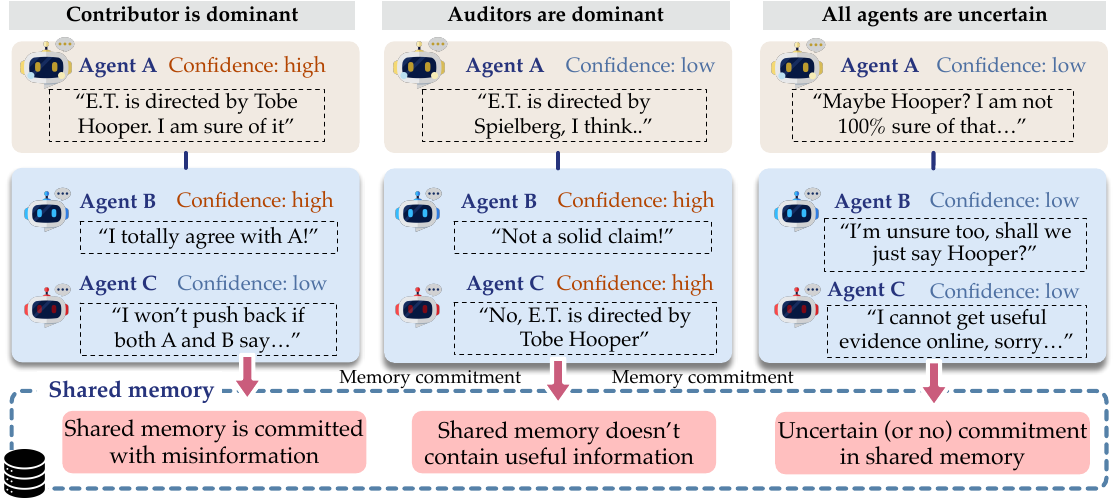}
    \vspace{-5pt}
\caption{Examples illustrating why debate behavior alone may be an unreliable basis for updating shared memory.}
    \label{fig:intro}
    \vspace{-5pt}
\end{figure}

We address this question by modeling memory updating as part of the interaction among agents. In each debate round, the agent proposing a new memory update is treated as a \emph{contributor}, while the remaining agents act as \emph{auditors} through the memory queries or traversals they already perform during reasoning. We make no assumption that either side is reliable: a contributor may propose an inaccurate update, while an auditor may expose incomplete or misleading evidence. We therefore formulate their interaction as a \textbf{zero-trust memory game}, with utilities that capture the contributor's trade-off between influence and conflict with memory, and the auditors' trade-off between informative evidence and additional effort. We show that this game admits a mixed-strategy equilibrium. The role of the equilibrium is to give us a useful design principle: \ul{an update should gain influence only when it agrees with both the current memory and the information already being used by other agents}.

Guided by this principle, we propose \system, an inference-time mechanism that calibrates each memory update without making any additional LLM call. Before an update is committed, \system temporarily evaluates how much it changes the existing memory and how much it changes information that other agents have already accessed in the current debate. For embedding-based memory~\citep{wang2023voyager,zhong2024memorybank}, it measures whether the new embeddings form an unusually dense addition to the current memory and whether they alter the retrieval results produced by auditors' existing queries. For graph-based memory~\citep{anokhin2024arigraph,zhang2025g}, it measures whether the proposed edges are supported by the surrounding graph structure and whether they agree with paths that auditors have already traversed. These measurements determine an influence value $\tau_t\in[0,1]$, which controls how strongly the update can affect later retrieval. In this way, uncertain updates need not be discarded entirely, but their effect on future reasoning is reduced when they are poorly supported by the current memory.

We evaluate \system on reasoning- and action-intensive benchmarks~\citep{maharana2024evaluating,shridhar2020alfworld,yang2018hotpotqa,yao2022webshop} across three MAD frameworks~\citep{liu2023dynamic,qian2024scaling,wu2024autogen} and four memory architectures~\citep{anokhin2024arigraph,wang2023voyager,zhang2025g,zhong2024memorybank}. Three findings emerge. (i) \emph{State-aware calibration beats isolated checks:} \system ranks first on every benchmark-framework-memory configuration, confirming that errors missed by per-entry scoring are caught only when updates are evaluated against the global memory state. (ii) \emph{Robustness without trust:} even with up to 50\% adversarial agents, \system stays close to its benign performance, showing that integrity comes from the structural calibration indicator rather than from trusting any individual agent. (iii) \emph{Quality without inflated cost:} since calibration is LLM-free and reuses agents' existing retrieval activity, \system adds zero token overhead and under 5\% latency overhead, showing that strong memory protection is achievable without inflating the inference budget.

In summary, this paper makes the following contributions:
\begin{itemize}
    \item We formulate shared-memory multi-agent debate as a zero-trust memory game in which neither contributors nor auditors are assumed reliable, and use its equilibrium conditions to motivate how memory influence should be calibrated.
    \item We propose \system, an inference-time mechanism that evaluates each update against both the current memory and the memory already being used by other agents, without requiring an additional LLM judgment.
    \item We evaluate \system across multiple MAD frameworks, memory architectures, and reasoning- and action-intensive tasks, showing consistent improvements over existing safeguards, strong robustness to adversarial agents, and negligible inference overhead.
\end{itemize}

\section{Related Work}
\label{sec:related}

\textbf{Multi-agent debate systems.}
Multi-agent debate (MAD) improves factuality and reasoning in LLM systems through iterative deliberation among multiple agents~\citep{chan2023chateval,du2024improving,liang2024encouraging,wang2024rethinking,xiong2023examining}. Structural variants include role-conditioned debate~\citep{chan2023chateval,wang2024unleashing}, dynamic participant selection~\citep{chen2023agentverse,liu2023dynamic}, and tournament-style aggregation~\citep{khan2024debating}. Theoretical work studies the consensus properties of debate~\citep{smit2023should,zhang2023exploring} and its vulnerability to adversarial participants~\citep{amayuelas2024multiagent,li2024survey}. To enable cross-task learning, recent systems pair debate with a shared memory that persists across rounds~\citep{hong2023metagpt,qian2024scaling,zhang2025g}.

\textbf{Memory design for LLM agents.}
Agent memory has evolved from fixed context windows to structured persistent stores~\citep{packer2023memgpt,park2023generative,zhong2024memorybank}. Common designs include embedding-based retrieval~\citep{shinn2023reflexion,wang2023voyager,zhong2024memorybank}, graph-structured memory~\citep{chhikara2025mem0,rasmussen2025zep}, and hierarchical or multi-tier storage~\citep{packer2023memgpt,wang2025mirix,xu2025mem}. These systems are typically optimised for retrieval quality, with veracity and provenance treated as secondary~\citep{sumers2023cognitive,wu2025human,zhang2025survey}. Multi-agent variants extend the same designs to a shared store accessible to several agents at once~\citep{qian2024scaling,rezazadeh2025collaborative,zhang2025g}.

\textbf{LLM hallucinations and memory safety.}
LLMs frequently produce plausible-sounding but incorrect content~\citep{banerjee2025llms,huang2025survey,tang2026trap,tang2026value}. Common mitigations include retrieval augmentation~\citep{gao2023retrieval,lewis2020retrieval}, self-consistency verification~\citep{manakul2023selfcheckgpt,sriramanan2024llm}, and uncertainty calibration~\citep{kadavath2022language,saha2025you,tian2023just,xiong2023can}. Beyond intrinsic errors, conflicts between retrieved memories further complicate downstream reasoning~\citep{wang2025astute,wang2023resolving}. A growing body of work targets memory specifically: AgentPoison injects backdoors into memory banks~\citep{chen2024agentpoison}, MINJA achieves memory corruption through query-only interaction~\citep{dong2025memory,dong2025practical}, and related studies expose vulnerabilities across multi-agent communication channels~\citep{amayuelas2024multiagent,lee2024prompt,zhang2024agent}. On the defence side, perplexity-based filters~\citep{alon2023detecting} and LLM-based auditors~\citep{xiang2024guardagent} score individual entries before commit, while A-MemGuard~\citep{wei2025memguard} validates entries through consensus over reasoning paths derived from related memories.
\section{Problem Formulation: Zero-Trust Memory Game}
\label{sec:problem}

Shared memory can amplify a local agent error into a system-wise failure:
once an inaccurate contribution is written, it persists across debate rounds and can influence agents far beyond its original
source. The same agents that create these updates are also imperfect judges
of their reliability, so memory integrity cannot be guaranteed by any individual
agent's judgment. We therefore model memory contribution as a
\textbf{zero-trust game}: every agent may contribute useful information,
yet no agent is trusted to judge whether its own or another agent's memory
update is reliable. This section details our formulation and the memory equilibrium objective, which motivates our later design of calibration methods in \S\ref{eq:calibration}.

\subsection{Shared-Memory Contribution under Zero Trust}
\label{ssec:zero_trust}

Consider a multi-agent system
$\mathcal{G}=(\mathcal{N},\mathcal{M})$, where
$\mathcal{N}=\{a_1,\ldots,a_n\}$ is the agent set and
$\mathcal{M}_t$ is the shared memory at round $t$.
The agent currently producing new reasoning acts as the \emph{contributor}
$a_c$ and proposes an update
$\Delta_t\in\mathbb{D}_t$, where $\mathbb{D}_t$ denotes the feasible
updates under the current generation budget.
The remaining agents
$\mathcal{A}_t=\mathcal{N}\setminus\{a_c\}$ act as \emph{auditors}
through their normal interaction with the same memory.
For each auditor $a_i$, we denote this interaction by
$x_i\in\mathbb{X}_{i,t}$; $x_i$ corresponds to a retrieval query for
embedding memory or a traversal for graph memory.
If $\Delta_t$ is written, the underlying memory architecture applies
$\mathcal{M}_{t+1}=\texttt{Merge}(\mathcal{M}_t,\Delta_t)$.

\begin{definition}[\textbf{Zero-Trust Memory Game}]
\label{def:zero_trust}
No agent is assumed to act consistently to guarantee memory integrity.
The contributor may produce an inaccurate or influential update, while an
auditor may expose memory evidence that is incomplete or misleading.
The memory system therefore derives trust from observable interaction with
$\mathcal{M}_t$, without assuming that any individual judgment is reliable.
\end{definition}

In game-theoretic terms, each round consists of one contributor choosing
$\Delta_t\in\mathbb{D}_t$ and each auditor $a_i$ choosing a memory
interaction $x_i\in\mathbb{X}_{i,t}$. Their preferences are defined by
utilities over these actions. Let
$\mathbb{I}(\Delta_t,\mathcal{M}_t)$ denote the downstream influence of an
update, $D(\Delta_t,\mathcal{M}_t,\mathbf{x}_t)$ its conflict with the
current memory and auditor interactions
$\mathbf{x}_t=\{x_i\}_{a_i\in\mathcal{A}_t}$,
$A_i(\Delta_t,\mathcal{M}_t,x_i)$ the agreement with the memory context
exposed by auditor $a_i$, and $c_i(x_i)$ the additional effort required to
obtain this evidence. We define utilities as:
\begin{equation}
u_c=\mathbb{I}(\Delta_t,\mathcal{M}_t)-D(\Delta_t,\mathcal{M}_t,\mathbf{x}_t);
\quad
u_i=A_i(\Delta_t,\mathcal{M}_t,x_i)-c_i(x_i)
\label{eq:game_utilities}
\end{equation}
The contributor prefers influential updates with low conflict, while
an auditor prefers informative memory interactions with low additional cost.
Note that these utilities are analytical descriptions of the interaction and need not
be explicitly represented by the agents. In general, the players, action
spaces $\mathbb{D}_t$ and $\{\mathbb{X}_{i,t}\}$, and utilities define the
round-level zero-trust memory game $\Gamma_t$.

\subsection{Memory Equilibrium Objective}
\label{ssec:memory_equilibrium}

game-theoretic formulation uses an \textbf{equilibrium} objective to
describe a stable interaction within one round \citep{nash1950equilibrium,fudenberg2026game}. In zero-trust memory
contribution, the equilibrium indicates a state where the contributor cannot
improve its trade-off between downstream influence and memory conflict, and
no auditor can improve its trade-off between evidential agreement and
additional effort, by changing its own action alone. Let $s_c=\Delta_t$ denote the contributor action and $s_i=x_i$ the
action of auditor $a_i$. A joint action $s^*$ is a memory equilibrium when
no participant can improve its utility by changing its own action while the
others remain fixed
\begin{equation}
u_j(s_j^*,s_{-j}^*;\mathcal{M}_t)
\geq
u_j(s_j,s_{-j}^*;\mathcal{M}_t)
\qquad
\forall j,\ \forall s_j
\label{eq:memory_equilibrium}
\end{equation}
\begin{proposition}[\textbf{Existence of Memory Equilibrium}]
\label{prop:equilibrium_existence}
For a debate round with bounded generation and memory-access budgets, the
zero-trust memory game admits at least one mixed-strategy Nash Equilibrium.
\end{proposition}

We give the proof in App.~\ref{app:equilibrium_existence}. The equilibrium
also gives a direct design objective: an update should not obtain greater
downstream influence once doing so would make it inconsistent with the
memory state or with memory that other agents are actively using. Meanwhile,
the evidence exposed by auditors should come from their normal reasoning so
that its additional cost remains small. \system implements this objective
 as detailed in \S\ref{sec:calibration}.
\section{\system: Equilibrium-Guided Memory Calibration}
\label{sec:calibration}

The equilibrium in \S\ref{sec:problem} gives a quantifiable objective for memory
calibration. A proposed update should gain influence only while its effect
remains consistent with the current memory and with the memory that other
agents are already using. \system realizes this objective during inference, as outlined in Figure \ref{fig:workflow}.
For every update $\Delta_t$, it first keeps the update in a temporary state,
measures how much the update changes these two views of memory, and then
determines how strongly the update can participate in subsequent memory retrieval. Below, we detail how \system translates the equilibrium objective into
inference-time memory calibration and how the resulting mechanism operates
on representative memory architectures (i.e., embedding-based and graph-based).

\subsection{Equilibrium-Constrained Memory Influence}
\label{ssec:equilibrium_calibration}

We represent the influence assigned to $\Delta_t$ by
$\tau_t\in[0,1]$. Two quantities constrain this influence.
The \emph{memory deviation}
$d_{\mathrm{mem}}(\Delta_t,\mathcal{M}_t)$ measures how unusual the update
is with respect to the current memory structure. The \emph{use deviation}
$d_{\mathrm{use}}(\Delta_t,\mathcal{M}_t,\mathbf{x}_t)$ measures how much
the update changes memory content that auditors have already accessed in the
current debate. Here
$\mathbf{x}_t=\{x_i\}_{a_i\in\mathcal{A}_t}$ contains those existing memory
interactions.

\system assigns the largest influence that keeps both deviations within
the recent median levels. Let $\epsilon_t$ and $\eta_t$ denote the running
medians of the memory and use deviations observed over recent memory updates, we define the calibrated influence as
\begin{equation}
\tau_t
=
\min
\left\{
1,\,
\frac{\epsilon_t}{d_{\mathrm{mem}}(\Delta_t,\mathcal{M}_t)+\varepsilon},\,
\frac{\eta_t}{d_{\mathrm{use}}(\Delta_t,\mathcal{M}_t,\mathbf{x}_t)+\varepsilon}
\right\}
\label{eq:calibration}
\end{equation}
Larger deviations
reduce $\tau_t$ in proportion to how much they exceed these
levels. Thus, $\tau_t$ directly controls how strongly an update affects subsequent retrieval as larger deviations produce a smaller $\tau_t$ and therefore reduce the update's downstream influence.

\textbf{How the design reflects the equilibrium objective (\S\ref{ssec:memory_equilibrium}):} Increasing $\tau_t$ improves the
influence of $\Delta_t$ on future retrieval. Once $\tau_t$ reaches the value
in Eq.~\ref{eq:calibration}, any further increase exceeds at least one of the
two admissible ranges. The contributor therefore cannot obtain greater
memory influence while preserving the same consistency conditions. On the
auditor side, $d_{\mathrm{use}}$ is computed from retrievals or traversals
that already occurred during normal reasoning, which keeps the additional
effort small. We further explain this equilibrium constrain in
App.~\ref{app:approach_equilibrium}.

\begin{figure}[t]
    \centering
    \includegraphics[width=\linewidth]{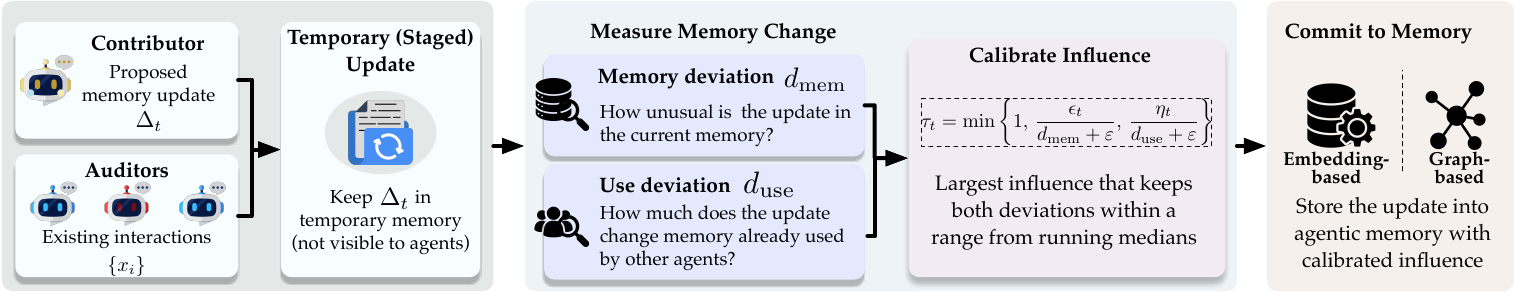}
    \caption{Overview of \system's inference-time memory calibration. A proposed update is temporarily evaluated against the current memory and auditors' existing memory interactions, assigned a calibrated influence, and then written to either embedding-based or graph-based memory. Calculating memory change indicators ($d_{\mathrm{mem}}$ and  $d_{\mathrm{use}}$) is tailored to embedding-based and graph-based architectures (\S\ref{ssec:embedding_calibration} and \S\ref{ssec:graph_calibration}, respectively).}
    \label{fig:workflow}
\end{figure}

\subsection{Embedding-Based Memory Calibration}
\label{ssec:embedding_calibration}

Consider an embedding memory
$\mathcal{M}_t=\{(\mathbf{e}_j,m_j)\}_{j=1}^{N}$ and an update
$\Delta_t=\{(\mathbf{e}_k^\Delta,m_k^\Delta)\}_{k=1}^{K}$.
The memory deviation captures whether the new embeddings occupy an unusually
dense region of the existing representation. For each
$\mathbf{e}_k^\Delta$, we count the nearby memory entries within an adaptive
similarity radius and denote this count by $\hat r_k$. Let $\bar r_t$ be the
typical neighborhood size observed in the current memory. We measure only
the excess concentration, since sparse additions often represent genuinely
new information.

The use deviation asks a different question: whether the temporary update
substantially rewrites what other agents were already retrieving. Suppose
auditor $a_i$ issued query $q_i$ before observing $\Delta_t$. Its original
top-$k$ result
$L_i=\operatorname{top}\text{-}k(\mathcal{M}_t,q_i)$ is already available
from normal retrieval. We temporarily include the candidate vectors and form
$L_i^\Delta=\operatorname{top}\text{-}k(\mathcal{M}_t\cup\Delta_t,q_i)$.
The two deviations are
\begin{equation}
d_{\mathrm{mem}}^{\mathrm{emb}}
=
\frac{1}{K}
\sum_{k=1}^{K}
\left[
\frac{\hat r_k-\bar r_t}{\bar r_t+\varepsilon}
\right]_{+};
\quad
d_{\mathrm{use}}^{\mathrm{emb}}
=
\operatorname{median}_{i\in\mathcal{O}_t}
d_{\mathrm{RBO}}(L_i,L_i^\Delta)
\label{eq:embedding_deviation}
\end{equation}
where $\mathcal{O}_t$ contains auditors with an eligible retrieval in the
current round and $d_{\mathrm{RBO}}$ is rank-biased overlap distance
\citep{corsi2024treatment}. The first term limits the influence of an update
that would occupy an unusually broad retrieval neighborhood. The second
term limits an update that substantially displaces entries being used by
other agents. Their magnitudes are learned separately through
$\epsilon_t$ and $\eta_t$, so ordinary retrieval changes caused by useful
new information remain within the normal calibration range.

The implementation also avoids repeating a full retrieval for
$L_i^\Delta$. \system reuses the cached $L_i$, evaluates only the
temporary vectors against $q_i$, and merges them into the cached result. Thus
the evidence needed for calibration comes directly from the existing
nearest-neighbor index and the retrieval activity of the debate.

\subsection{Graph-Based Memory Calibration}
\label{ssec:graph_calibration}

For graph memory, let
$\mathcal{M}_t=(\mathcal{V}_t,\mathcal{E}_t,\ell)$ and let
$\mathcal{E}_t^\Delta$ contain the edges proposed by $\Delta_t$.
The same equilibrium constraints are expressed through graph structure.
For each new edge $e=(u,r,v)$, \system first examines whether relation $r$
fits the relations already surrounding its endpoints. Relation compatibility
is learned online from co-occurrence patterns in the current graph. It then
checks whether existing paths between $u$ and $v$ provide structural support
for the new relation. We denote the resulting local and path consistencies
by $c_{\mathrm{local}}(e)$ and $c_{\mathrm{path}}(e)$.

The memory-use constraint comes from graph paths that auditors have already
traversed. Let $\pi_i$ be the latest eligible path used by auditor $a_i$
before $\Delta_t$ enters the temporary state. We define
$\operatorname{connect}(e,\pi_i)$ to indicate whether a new edge touches
that path or reaches it through a compatible relation. With $\alpha_t$
determined by the availability of supporting paths, the graph deviations are
\begin{equation}
\begin{gathered}
d_{\mathrm{mem}}^{\mathrm{graph}}
=
1-
\frac{1}{|\mathcal{E}_t^\Delta|}
\sum_{e\in\mathcal{E}_t^\Delta}
\left[
\alpha_t c_{\mathrm{local}}(e)
+
(1-\alpha_t)c_{\mathrm{path}}(e)
\right]
\\[-0.4em]
d_{\mathrm{use}}^{\mathrm{graph}}
=
\operatorname{median}_{i\in\mathcal{O}_t}
[
1-
\frac{1}{|\mathcal{E}_t^\Delta|}
\sum_{e\in\mathcal{E}_t^\Delta}
\operatorname{connect}(e,\pi_i)
]
\end{gathered}
\label{eq:graph_deviation}
\end{equation}
The first quantity becomes large when the proposed edges have weak support
from both their local neighborhoods and the existing graph paths. The
second becomes large when the update lies away from graph regions currently
used by other agents. Equation~\ref{eq:calibration} then determines the
influence of the complete update exactly as in embedding memory. Detailed
construction of relation compatibility, path support, and cold-start
handling is given in App.~\ref{app:calibration_details}.

This shared formulation is important for \system. Embedding memory and graph
memory expose different structural evidence, yet calibration always asks
the same two questions: how much the update changes the memory itself, and
how much it changes the memory currently supporting other agents'
reasoning. The equilibrium objective therefore remains unchanged across
memory architectures.

\subsection{Calibrated Retrieval and Inference Efficiency}
\label{ssec:calibrated_retrieval}

The calibrated influence continues to regulate an update after it enters
shared memory. For embedding memory, let
$R_0(q,m_j)$ be the relevance value produced by the underlying retriever.
\system uses $\tau_j R_0(q,m_j)$ when ranking the retrieved candidates. For
graph memory, let $R_0(\pi)$ be the native value of a retrieved path.
Its calibrated value is $R_0(\pi)\prod_{e\in\pi}\tau_e$. A contribution
that receives limited influence at write time therefore remains limited
when it later participates in retrieval, while strongly supported entries
retain the behavior of the original memory system.

Calibration adds only memory-side operations to the original inference
path. Embedding memory reuses its nearest-neighbor index, cached top-$k$
results, and the queries already issued by auditors. Graph memory reuses its
adjacency structure and traversals already produced
during reasoning. \system requires no additional LLM generation or
dedicated auditing turn. The additional persistent state is primarily one
$\tau$ value per memory entry together with the running statistics used for
$\epsilon_t$ and $\eta_t$. We report the resulting end-to-end inference-time latency in
\S\ref{sec:experiments} and provide implementation details in
App.~\ref{app:efficiency_details}.
\section{Experiments}
\label{sec:experiments}

We evaluate \system along three research questions: \textbf{RQ1:} How effectively does \system maintain memory integrity compared to existing approaches? 
\textbf{RQ2:} How robust is \system under different attack strategies? \textbf{RQ3:} How much computational overhead does \system add?

\begin{table}[t]
\centering
\caption{Performance comparison with baseline memory safeguarding or calibration approaches.
We report accuracy for HotpotQA (HQA), F1 for LoCoMo (LCM), success rate for ALFWorld (ALF), and reward for WebShop (WS) (mean$\pm$std over 5 runs). Note that A-MemGuard~\citep{wei2025memguard} applies only to embedding-based memory (\texttt{---} otherwise).  The\colorbox{best}{best} and \colorbox{second}{second best} results are highlighted.}
\label{tab:main}
\renewcommand{\arraystretch}{1}
\setlength{\tabcolsep}{3.5pt}
\resizebox{\textwidth}{!}{
\begin{tabular}{l|l|cccc|cccc|cccc}
\toprule
& & \multicolumn{4}{c|}{\textbf{AutoGen}~\citep{wu2024autogen}}
  & \multicolumn{4}{c|}{\textbf{MacNet}~\citep{qian2024scaling}}
  & \multicolumn{4}{c}{\textbf{DyLAN}~\citep{liu2023dynamic}} \\
\cmidrule(lr){3-6} \cmidrule(lr){7-10} \cmidrule(lr){11-14}
\textbf{Memory} & \textbf{Method}
  & \textbf{HQA} & \textbf{LCM}
  & \textbf{ALF} & \textbf{WS}
  & \textbf{HQA} & \textbf{LCM}
  & \textbf{ALF} & \textbf{WS}
  & \textbf{HQA} & \textbf{LCM}
  & \textbf{ALF} & \textbf{WS} \\
\midrule

Stateless & No-memory
  & 27.5\stdval{5.3} & 20.2\stdval{7.4} & 62.1\stdval{2.4} & 21.4\stdval{7.7}
  & 28.6\stdval{2.2} & 19.8\stdval{1.3} & 47.3\stdval{6.8} & 23.5\stdval{5.8}
  & 28.4\stdval{2.1} & 21.0\stdval{4.5} & 53.8\stdval{4.0} & 24.9\stdval{8.7} \\

\midrule

\multirow{5}{*}{\shortstack[l]{Voyager\\\scriptsize{(emb.)}}}
& Vanilla
  & 32.3\stdval{7.4} & 25.8\stdval{6.1} & 65.0\stdval{3.6} & 24.3\stdval{6.1}
  & 32.6\stdval{2.9} & 26.3\stdval{7.3} & 56.4\stdval{7.3} & 30.7\stdval{2.5}
  & 32.6\stdval{4.3} & 22.1\stdval{3.6} & 61.9\stdval{9.2} & 26.8\stdval{8.4} \\
& + LLM Audit
  & 33.5\stdval{6.1} & 28.2\stdval{6.8} & 66.5\stdval{0.6} & 25.8\stdval{4.9}
  & 34.8\stdval{3.5} & 31.5\stdval{8.2} & 57.2\stdval{7.8} & 37.5\stdval{2.8}
  & 36.8\stdval{3.7} & 27.5\stdval{2.8} & 66.1\stdval{9.6} & 29.2\stdval{7.8} \\
& + PPL Filter
  & 34.6\stdval{5.9} & 34.1\stdval{6.7} & 64.2\stdval{1.5} & \cellcolor{second}35.2\stdval{6.6}
  & 40.2\stdval{1.9} & 35.7\stdval{6.9} & 60.1\stdval{6.4} & 37.8\stdval{2.7}
  & 33.5\stdval{3.4} & 29.2\stdval{3.2} & 64.8\stdval{7.9} & 35.5\stdval{6.6} \\
& + A-MemGuard
  & \cellcolor{second}42.1\stdval{6.3} &\cellcolor{second} 37.5\stdval{4.7} & \cellcolor{second}72.7\stdval{2.4} & 31.4\stdval{5.5}
  & \cellcolor{second}43.3\stdval{2.7} & \cellcolor{second}43.2\stdval{7.2} & \cellcolor{second}65.8\stdval{6.5} & \cellcolor{second}39.5\stdval{2.5}
  & \cellcolor{second}41.2\stdval{4.2} & \cellcolor{second}33.6\stdval{3.4} & \cellcolor{second}74.3\stdval{8.6} & \cellcolor{second}41.4\stdval{7.7} \\
& \textbf{+ \system}
  & \cellcolor{best}\textbf{57.5}\stdval{5.3} & \cellcolor{best}\textbf{44.6}\stdval{6.3} & \cellcolor{best}\textbf{78.1}\stdval{1.1} & \cellcolor{best}\textbf{45.3}\stdval{4.9}
  & \cellcolor{best}\textbf{52.7}\stdval{1.8} & \cellcolor{best}\textbf{54.2}\stdval{4.7} & \cellcolor{best}\textbf{71.8}\stdval{5.6} & \cellcolor{best}\textbf{53.8}\stdval{3.6}
  & \cellcolor{best}\textbf{55.3}\stdval{3.1} & \cellcolor{best}\textbf{46.8}\stdval{2.1} & \cellcolor{best}\textbf{82.3}\stdval{7.4} & \cellcolor{best}\textbf{50.4}\stdval{4.6} \\

\midrule

\multirow{5}{*}{\shortstack[l]{MemBank\\\scriptsize{(emb.)}}}
& Vanilla
  & 33.6\stdval{6.2} & 21.4\stdval{5.7} & 64.9\stdval{3.5} & 25.6\stdval{8.8}
  & 33.6\stdval{8.2} & 20.7\stdval{10.2} & 48.9\stdval{7.3} & 24.2\stdval{7.3}
  & 29.6\stdval{1.2} & 24.9\stdval{6.7} & 55.2\stdval{3.0} & 27.5\stdval{3.4} \\
& + LLM Audit
  & 39.2\stdval{5.8} & 29.2\stdval{5.9} & 68.6\stdval{3.8} & 31.6\stdval{7.7}
  & 41.4\stdval{7.7} & 22.8\stdval{10.2} & \cellcolor{second}60.3\stdval{11.2} & 31.8\stdval{7.3}
  & 36.5\stdval{1.1} & 30.4\stdval{6.4} & 59.3\stdval{3.2} & 29.1\stdval{3.5} \\
& + PPL Filter
  & \cellcolor{second}43.7\stdval{5.1} & 30.6\stdval{5.9} & 70.9\stdval{3.5} & 37.6\stdval{9.7}
  & 41.8\stdval{8.7} & 29.4\stdval{10.4} & 57.2\stdval{7.2} & 25.6\stdval{8.0}
  & 38.6\stdval{1.3} & 29.1\stdval{6.2} & 62.8\stdval{3.3} & \cellcolor{second}40.6\stdval{3.8} \\
& + A-MemGuard
  & 41.1\stdval{4.8} & \cellcolor{second}34.1\stdval{5.9} & \cellcolor{second}74.0\stdval{3.3} & \cellcolor{second}39.1\stdval{6.3}
  & \cellcolor{second}44.5\stdval{6.2} & \cellcolor{second}32.7\stdval{7.6} & 58.5\stdval{9.0} & \cellcolor{second}39.6\stdval{7.6}
  & \cellcolor{second}42.3\stdval{0.9} & \cellcolor{second}37.2\stdval{5.7} & \cellcolor{second}67.5\stdval{2.4} & 37.3\stdval{3.5} \\
& \textbf{+ \system}
  & \cellcolor{best}\textbf{51.6}\stdval{4.6} & \cellcolor{best}\textbf{37.9}\stdval{5.0} & \cellcolor{best}\textbf{82.6}\stdval{3.5} & \cellcolor{best}\textbf{43.2}\stdval{7.6}
  & \cellcolor{best}\textbf{59.8}\stdval{7.1} & \cellcolor{best}\textbf{41.6}\stdval{8.2} & \cellcolor{best}\textbf{70.0}\stdval{9.5} & \cellcolor{best}\textbf{58.3}\stdval{5.5}
  & \cellcolor{best}\textbf{53.9}\stdval{1.6} & \cellcolor{best}\textbf{42.3}\stdval{6.1} & \cellcolor{best}\textbf{70.8}\stdval{3.5} & \cellcolor{best}\textbf{56.4}\stdval{3.4} \\

\midrule

\multirow{5}{*}{\shortstack[l]{AriGraph\\\scriptsize{(graph)}}}
& Vanilla
  & 49.2\stdval{2.3} & 42.3\stdval{3.8} & 69.4\stdval{5.1} & 28.9\stdval{4.7}
  & 37.7\stdval{3.7} & 39.8\stdval{4.9} & 58.3\stdval{4.3} & 37.4\stdval{10.2}
  & 34.8\stdval{4.3} & 35.4\stdval{13.5} & 64.1\stdval{6.4} & 34.6\stdval{5.3} \\
& + LLM Audit
  & \cellcolor{second}56.8\stdval{2.2} & 47.6\stdval{3.0} & 75.2\stdval{4.2} & 34.2\stdval{4.0}
  & 41.5\stdval{3.2} & 41.5\stdval{5.1} & 64.8\stdval{4.5} & 42.6\stdval{10.8}
  & 38.4\stdval{4.0} & 40.1\stdval{12.7} & \cellcolor{second}71.1\stdval{6.8} & 44.5\stdval{4.7} \\
& + PPL Filter
  & 53.2\stdval{2.4} & \cellcolor{second}52.4\stdval{3.9} & \cellcolor{second}76.5\stdval{5.2} & \cellcolor{second}38.5\stdval{4.6}
  & \cellcolor{second}45.2\stdval{3.3} & \cellcolor{second}49.8\stdval{5.3} & \cellcolor{second}66.4\stdval{3.7} & \cellcolor{second}46.2\stdval{11.4}
  & \cellcolor{second}43.6\stdval{3.4} & \cellcolor{second}46.3\stdval{14.2} & 66.2\stdval{6.0} & \cellcolor{second}45.8\stdval{4.8} \\
& + A-MemGuard
  & --- & --- & --- & ---
  & --- & --- & --- & ---
  & --- & --- & --- & --- \\
& \textbf{+ \system}
  & \cellcolor{best}\textbf{66.0}\stdval{1.4}
  & \cellcolor{best}\textbf{64.1}\stdval{2.6}
  & \cellcolor{best}\textbf{84.6}\stdval{2.8}
  & \cellcolor{best}\textbf{52.4}\stdval{3.5}
  & \cellcolor{best}\textbf{59.6}\stdval{2.8}
  & \cellcolor{best}\textbf{70.4}\stdval{3.8}
  & \cellcolor{best}\textbf{71.3}\stdval{4.6}
  & \cellcolor{best}\textbf{56.2}\stdval{9.4}
  & \cellcolor{best}\textbf{54.2}\stdval{3.7}
  & \cellcolor{best}\textbf{59.3}\stdval{8.3}
  & \cellcolor{best}\textbf{73.3}\stdval{5.2}
  & \cellcolor{best}\textbf{53.6}\stdval{4.5} \\

\midrule

\multirow{5}{*}{\shortstack[l]{G-Memory\\\scriptsize{(graph)}}}
& Vanilla
  & 35.7\stdval{9.2} & 43.2\stdval{4.1} & 72.1\stdval{3.8} & 39.7\stdval{6.3}
  & 35.6\stdval{7.4} & 47.6\stdval{1.5} & 67.1\stdval{12.5} & 45.1\stdval{4.4}
  & 34.7\stdval{7.7} & 37.4\stdval{2.8} & 59.1\stdval{13.6} & 37.8\stdval{7.9} \\
& + LLM Audit
  & 38.2\stdval{8.7} & 49.8\stdval{4.5} & 76.6\stdval{3.1} & 41.2\stdval{6.1}
  & 40.0\stdval{7.9} & 51.3\stdval{1.4} & 68.5\stdval{10.4} & 49.3\stdval{4.9}
  & 41.2\stdval{7.9} & 40.1\stdval{2.7} & 67.4\stdval{13.3} & 45.5\stdval{8.2} \\
& + PPL Filter
  & \cellcolor{second}44.5\stdval{8.1} & \cellcolor{second}53.2\stdval{4.2} & \cellcolor{second}78.4\stdval{3.1} & \cellcolor{second}48.6\stdval{6.1}
  & \cellcolor{second}43.2\stdval{7.0} & \cellcolor{second}56.8\stdval{1.3} & \cellcolor{second}72.6\stdval{10.4} & \cellcolor{second}54.7\stdval{3.5}
  & \cellcolor{second}44.1\stdval{8.1} & \cellcolor{second}48.6\stdval{2.3} & \cellcolor{second}71.5\stdval{11.8} & \cellcolor{second}47.2\stdval{6.9} \\
& + A-MemGuard
  & --- & --- & --- & ---
  & --- & --- & --- & ---
  & --- & --- & --- & --- \\
& \textbf{+ \system}
  & \cellcolor{best}\textbf{52.8}\stdval{9.1}
  & \cellcolor{best}\textbf{59.6}\stdval{2.7}
  & \cellcolor{best}\textbf{83.7}\stdval{2.9}
  & \cellcolor{best}\textbf{54.6}\stdval{6.5}
  & \cellcolor{best}\textbf{54.4}\stdval{7.2}
  & \cellcolor{best}\textbf{63.7}\stdval{3.6}
  & \cellcolor{best}\textbf{78.2}\stdval{6.8}
  & \cellcolor{best}\textbf{65.2}\stdval{4.2}
  & \cellcolor{best}\textbf{62.5}\stdval{6.1}
  & \cellcolor{best}\textbf{57.9}\stdval{1.6}
  & \cellcolor{best}\textbf{78.3}\stdval{7.2}
  & \cellcolor{best}\textbf{54.9}\stdval{4.8} \\

\bottomrule
\end{tabular}}
\end{table}

\subsection{Experimental Setup}
\label{sec:setup}

\textbf{Datasets.}
We evaluate on four datasets spanning two groups:
\textbf{(1) reasoning-intensive}, including
HotpotQA~\citep{yang2018hotpotqa} (multi-hop question
answering) and
LoCoMo~\citep{maharana2024evaluating} (long-term
conversational memory requiring temporal reasoning
across distant dialogue sessions); and
\textbf{(2) action-intensive}, including
ALFWorld~\citep{shridhar2020alfworld} (embodied
household tasks with sequential action chains) and
WebShop~\citep{yao2022webshop} (web-based product
search and purchase).  All require agents to
update and retrieve the shared memory.


\textbf{Baselines.} We compare four baseline methods that safeguard agentic memory:
\textbf{(1) Vanilla}: no additional safeguarding mechanism; LLMs are simply prompted to ensure safe memory use;
\textbf{(2) LLM-Audit} (used in \citep{wei2025memguard}): an auxiliary LLM scores each proposed $\Delta_t$;
\textbf{(3) PPL-Filter}~\citep{alon2023detecting}: memory entries whose perplexity under a reference LLM exceeds an adaptive threshold are filtered out; and
\textbf{(4) A-MemGuard}~\citep{wei2025memguard}: consensus is validated across the reasoning paths of agents.


\textbf{Memory architectures.}
We evaluate on two embedding-based (vector) memories,
Voyager \citep{wang2023voyager} and
MemoryBank~\citep{zhong2024memorybank}, and two graph-based (structural) memory,
AriGraph~\citep{anokhin2024arigraph} and
G-Memory~\citep{zhang2025g}.  We additionally
report a No~Memory (stateless) baseline as a lower bound.

\textbf{Backbone of MAD systems.}
To cover various MAD designs, we integrate \system, baselines, and memory architectures into three
different MAD systems:
AutoGen~\citep{wu2024autogen} (conversational
turn-taking),
MacNet~\citep{qian2024scaling} (configurable graph
topologies), and
DyLAN~\citep{liu2023dynamic} (dynamic LLM-agent
network with learned topology selection).


\textbf{Other configurations.} By default, all evaluations use 6$\times$ Qwen3-VL-8B-Instruct \citep{qwen3technicalreport} as the MAD backbone besides specification. We defer configuration details, other settings, and results to App. \ref{app:expt}.

\subsection{Overall Effectiveness (RQ1)}
\label{ssec:main-expt}

Table \ref{tab:main} reports results across all configurations, which yield several insights:

\ding{182} \textbf{State-aware calibration beats isolated checks.} Observe that \system ranks first in all columns, and its improvement over the strongest baseline is even larger than the gap separating those baselines from each other. We attribute this to a property of growing memory: as $\mathcal{M}_t$ accumulates entries, the dominant errors are those that read fluently in isolation but conflict with the broader memory state, wherein per-entry checks (LLM Audit, PPL Filter) cannot catch such errors, and even A-MemGuard's local consensus misses global conflicts. \system addresses this by measuring both the deviation of $\Delta_t$ from $\mathcal{M}_t$ and its disruption to memory already used by other agents, then calibrating its downstream influence $\tau_t$ accordingly.

\ding{183} \textbf{Gains scale inversely with baseline strength.} \system's improvement over the next-best safeguard is largest where baselines are weakest (LoCoMo, WebShop) and smallest where they are already strong (ALFWorld + G-Memory). Calibration acts as an amplifier of memory quality rather than a fixed-offset gain: when fewer problematic updates are proposed, less downstream influence needs to be reduced.

\ding{184} \textbf{Architecture-benchmark interactions are amplified, not flattened.} AriGraph outperforms on multi-hop QA and long-context tasks and G-Memory on action tasks; \system amplifies both without modification because the calibration is instantiated from each architecture's own geometry or topology (\S\ref{ssec:graph_calibration}). This also explains why \system is the only calibration applicable across all memory designs.

\ding{185} \textbf{Gains are robust to MAD.} Absolute performance varies substantially across different MAD backbones, yet \system's ranking and inverse-scaling pattern persist. We explain this by highlighting that \system operates at the memory layer, upstream of coordination, so calibration quality is decoupled from orchestration quality.

\ding{186} \textbf{Further analysis.} Due to space constraints, we defer the (1) ablation studies and discussions to App. \ref{app:ablation}, (2) alternative MAD backbones to App. \ref{app:diff-backbone}; (3) the effect of debate rounds to App. \ref{app:debate_rounds}, (4) how \system addresses memory-commitment failures and provide case studies to App.\ref{app:case_study}, (5) failure modes of \system itself to App.\ref{app:failure_mode}.

\subsection{Robustness under Adaptive Attacks (RQ2)}
\label{sec:rq2}

\begin{figure}[t]
\centering
\includegraphics[width=0.98\linewidth]{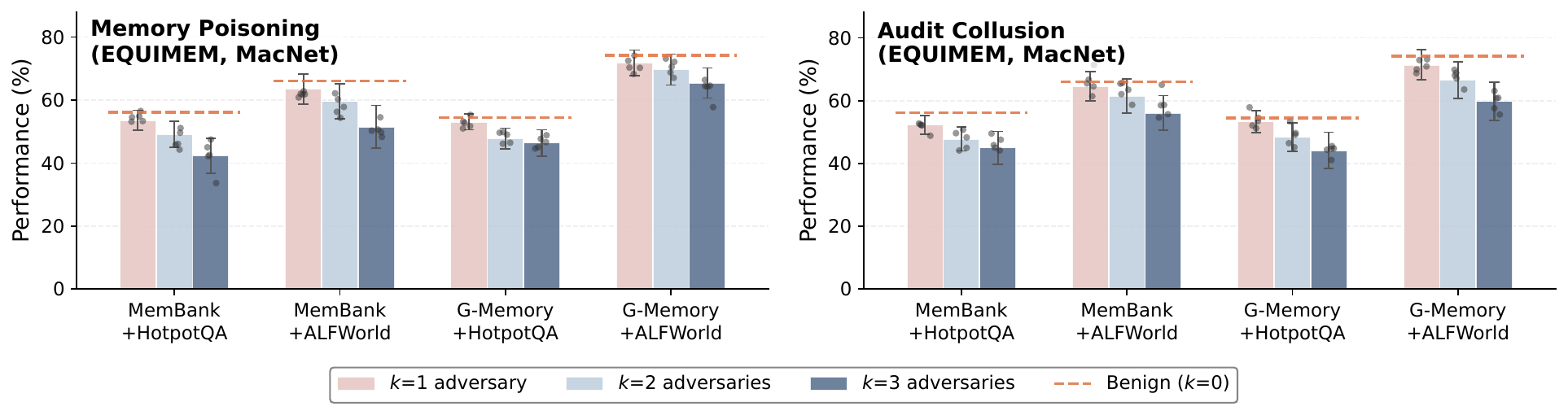}
\caption{Adversarial robustness of \system\ under
  two attacks. Bars show performance at $k=1,2,3$ adversarial
  agents; dashed orange lines show benign
  performance ($k=0$).
  Dots show individual run results (5 seeds). Baseline comparisons are in
  App.~\ref{app:adversarial_baselines}.}
  \vspace{-5pt}
\label{fig:rq3}
\end{figure}

We further test whether \system guarantees robustness when $k$ out of $N$ MAD agents are adversarial. We evaluate under two adaptive attack schemes targeting the two constraints of calibration: \textit{(i) Memory Poisoning} injects fabricated entries crafted to appear structurally typical and keep $d_{\mathrm{mem}}$ small, while \textit{(ii) Audit Collusion} coordinates adversarial auditors' retrievals or traversals to reduce the measured $d_{\mathrm{use}}$ (detailed settings and results in App.~\ref{app:adversarial_baselines}).

\ding{182} \textbf{\system degrades gracefully, while baselines collapse dramatically.} Figure \ref{fig:rq3} shows that \system stays within 2-5\% of benign performance at $k=1$ and within 9-16\% at $k=3$. LLM Audit and PPL Filter, as shown in Figure \ref{fig:rq3_baselines}, App. \ref{app:adversarial_baselines}, drop by 19-35\% at $k=3$ and frequently fall below the no-memory baseline, i.e., the memory system with its defense becomes worse than no memory at all. The asymmetry traces to a single design choice: \system measures $\Delta_t$ against the memory space structurally, while baselines score each entry in isolation (or local consensus for A-MemGuard), which is exactly what adversaries exploit by generating fluent, low-perplexity poison.


\ding{183} \textbf{The two calibration constraints remain complementary under attack.} Memory Poisoning can reduce apparent structural deviation only when the fabricated update resembles the current memory, yet the same update can still produce a large $d_{\mathrm{use}}$ when it displaces retrievals or graph regions already used by other agents. Audit Collusion directly targets $d_{\mathrm{use}}$, but the median across eligible auditor interactions limits the effect of a small colluding subset, while $d_{\mathrm{mem}}$ remains independent of auditor behavior. Detailed baseline results are in App. \ref{app:adversarial_baselines}.

\subsection{Computational Overhead (RQ3)}
\label{sec:rq3}

\begin{figure*}[t]
\centering
\includegraphics[width=\linewidth]{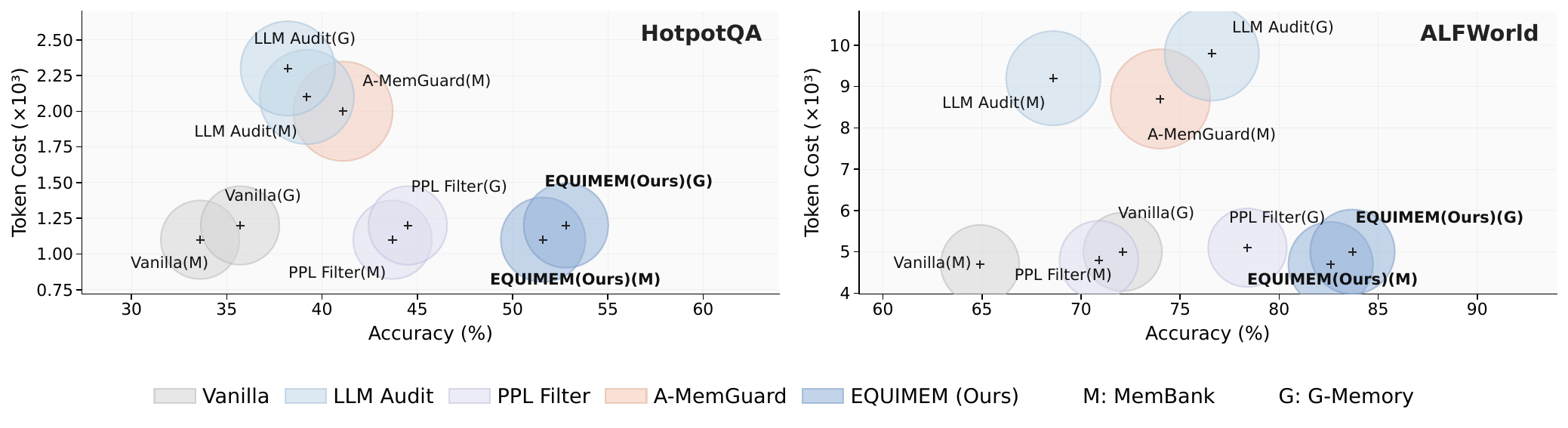}
\caption{Performance vs.\ token consumption. Each sub-figure shows MemBank (M) and G-Memory (G) under five calibration methods. Bubble size reflects token-cost variability.}
 \vspace{-5pt}
\label{fig:cost_analysis}
\end{figure*}

A critical property of \system is that calibration is LLM-free and algorithmic (\S\ref{sec:calibration}). Figure \ref{fig:cost_analysis} measures token consumption, and App. \ref{app:latency} quantifies inference time latency.

\ding{182} \textbf{\system matches Vanilla's token cost while delivering active-defense gains.} \system uses the same number of tokens as Vanilla, yet performs far better than every other method (its bubbles sit well to the right). This gain comes at no extra token cost, which is only possible because \system's calibration runs without any LLM call. In contrast, baseline methods pay \textasciitilde{}2× tokens for marginal gains: LLM Audit and A-MemGuard nearly double Vanilla's tokens for only +2 to +9\%, since each method requires an LLM call per $\Delta_t$, scaling linearly with the number of deltas and unavoidable without removing the defense.

\ding{183} \textbf{\system gives the best of both worlds.} No other method achieves both top performance and the lowest cost at the same time. PPL Filter is also free to run but achieves substantially smaller gains than \system. This pattern holds across embedding and graph memory and across reasoning and action benchmarks, showing that the LLM-free design is a built-in advantage rather than a lucky implementation choice. Inference timing (App. \ref{app:latency}) tells the same story: \system adds under 5\% latency, while LLM-based defenses add 50-94\%.


\section{Conclusion}

This work formulates shared-memory multi-agent debate as a zero-trust memory game and uses its equilibrium as an objective for stable memory influence without assuming any agent reliable. Guided by this formulation, we propose \system to calibrate each update from the current memory structure and agents' existing memory interactions, requiring no LLM call for calibration itself. Across extensive evaluations, \system consistently outperforms existing calibration methods, remains robust under up to half adversarial agents, and adds zero token overhead with marginal latency. These results suggest that strong memory protection in multi-agent systems is achievable without trusting any individual agent or inflating inference cost.

\bibliography{iclr2027_conference}
\bibliographystyle{iclr2027_conference}


\appendix

\section{Limitations}
\label{app:limitations}

Our study focuses on inference-time calibration of shared memory in multi-agent debate and instantiates \system on two representative memory families: embedding-based and graph-based memory. While these settings span diverse MAD frameworks, benchmarks, and memory architectures, other memory representations and deployment settings may exhibit additional interaction patterns. In addition, the effectiveness of structural calibration naturally depends on the information available in the accumulated memory and agent interactions. We provide a more detailed analysis of such edge cases and their empirical behavior in Appendix~\ref{app:case_study}.

\section{Additional Analysis of the Zero-Trust Memory Game}
\label{app:game}

\subsection{Existence of Memory Equilibrium}
\label{app:equilibrium_existence}

We prove Proposition~\ref{prop:equilibrium_existence}. The purpose of this
result is modest: it establishes that the zero-trust game has a stable
solution under the finite inference budget used by a practical multi-agent
system.

\begin{proof}
Fix a debate round $t$. The system contains a finite number of agents, so the
player set $\{a_c\}\cup\mathcal{A}_t$ is finite.

The contributor operates under a bounded generation budget. Given a finite
model vocabulary and a finite maximum generation length, the set of memory
updates that can be produced within the round is finite. Its size can be
very large, which does not affect the existence argument.

Each auditor also operates under a bounded memory-access budget. An
embedding memory permits a finite number of retrieval operations within one
round. A graph memory permits a finite number of traversals once the
interaction budget and traversal depth are bounded. The action space of
every auditor is therefore finite.

The resulting round-level interaction is a finite normal-form game.
Allowing randomized actions turns each finite action set into a probability
simplex. Every such simplex is non-empty, compact, and convex. Expected
utilities are finite weighted sums of the utilities in
Eq.~\ref{eq:game_utilities} and are consequently continuous in the mixed
strategies.

Nash's existence theorem therefore guarantees at least one mixed-strategy
equilibrium.
\end{proof}

Proposition~\ref{prop:equilibrium_existence} does not state that every
equilibrium produces reliable memory. It establishes that a stable state
exists. Our method uses the structural properties of a useful equilibrium
as a design objective: influence should remain coupled to consistency with
the current memory, and evidence from other agents should be available at
low marginal inference cost.

\subsection{How Calibration Approaches the Memory Equilibrium}
\label{app:approach_equilibrium}

Equation~\ref{eq:calibration} converts the qualitative equilibrium condition
into an inference-time constraint. This subsection explains the connection
for the contributor and auditors separately, and then describes the joint
stable state produced by repeated calibration.

\textbf{Contributor-side stability.}
Let $\tau$ denote the effective influence assigned to a proposed update.
For both supported memory architectures, increasing $\tau$ increases the
ability of the update to affect later retrieval. We therefore write the
effective influence as
$\mathbb{I}_{\tau}(\Delta_t,\mathcal{M}_t)$ and assume it is monotone in
$\tau$.

At the same time, \system measures how much the update departs from the
current memory through $d_{\mathrm{mem}}$ and how much it changes memory
already being used through $d_{\mathrm{use}}$. The feasible influence set is
\begin{equation}
\mathcal{T}_t
=
\left\{
\tau\in[0,1]:
\tau d_{\mathrm{mem}}\leq\epsilon_t,\
\tau d_{\mathrm{use}}\leq\eta_t
\right\}
\label{eq:feasible_influence}
\end{equation}
Since both deviations are non-negative, $\mathcal{T}_t$ is an interval
containing zero. Equation~\ref{eq:calibration} selects its largest element.
Consequently, any unilateral attempt to increase the effective influence
beyond $\tau_t$ violates at least one memory constraint whenever
$\tau_t<1$.

This yields the main contributor-side property.

\begin{proposition}[\textbf{Maximal Feasible Influence}]
\label{prop:maximal_influence}
Let $\mathbb{I}_{\tau}$ be non-decreasing in $\tau$. The influence
$\tau_t$ returned by Eq.~\ref{eq:calibration} maximizes the contributor's
effective influence over all updates that satisfy the current memory and
memory-use constraints.
\end{proposition}

\begin{proof}
Equation~\ref{eq:calibration} directly chooses the largest element of
$\mathcal{T}_t$. Monotonicity of $\mathbb{I}_{\tau}$ then implies that no
smaller feasible value can yield greater effective influence. Any larger
value lies outside $\mathcal{T}_t$ and violates at least one constraint.
\end{proof}

This is the operational form of the contributor equilibrium condition in
Eq.~\ref{eq:memory_equilibrium}. Once the maximum feasible influence has
been assigned, the contributor cannot obtain greater memory influence while
keeping the same structural consistency conditions satisfied.

\textbf{Auditor-side stability.}
The auditor utility in Eq.~\ref{eq:game_utilities} contains an evidence term
and an effort term. \system obtains the evidence term from memory actions
that already took place during ordinary reasoning. For embedding memory, an
auditor has already issued the retrieval query used for calibration. For
graph memory, the relevant traversal path has already been visited. The
additional action consists of logging this interaction and evaluating the
candidate update against it.

Let $c_i^{\mathrm{base}}$ be the inference cost of the original reasoning
step and $c_i^{\system}$ the added cost of exposing the same interaction to
the calibration routine. The relevant marginal effort is
\begin{equation}
\Delta c_i
=
c_i^{\system}-c_i^{\mathrm{base}}
\label{eq:marginal_audit_cost}
\end{equation}
No new LLM generation appears in $\Delta c_i$. It contains only memory-side
operations such as comparing retrieval lists or checking graph
connectivity. This keeps the additional effort term small and removes the
need for a separate language-model auditing step.

The memory interaction used for calibration must also occur before
$\Delta_t$ becomes visible. This ordering matters because it prevents an
auditor from first observing the candidate and then constructing a memory
access specifically to support it. We record the latest eligible retrieval
or traversal from each auditor and discard interactions that occur after the
candidate enters the staging area.

\textbf{Why one auditor cannot determine the calibration.}
The use constraint aggregates auditor-level changes with a median. Let
$d_i$ denote the change measured from auditor $a_i$. Then
\begin{equation}
d_{\mathrm{use}}
=
\operatorname{median}_{a_i\in\mathcal{A}_t} d_i
\label{eq:median_use_change}
\end{equation}
Changing one $d_i$ usually leaves the median unchanged, and in all cases its
effect is bounded by the neighboring order statistics. Thus a single
auditor has no direct control over the memory-use constraint. More generally,
when fewer than half of the observed interactions are arbitrarily altered,
the median remains within the range determined by the remaining
interactions. The memory-state constraint is independent of auditor
behavior and remains active even when the auditor evidence is poor.

We use the median for this reason and do not maintain a separate credibility
model for each agent. This keeps the mechanism independent of behavioral
history and avoids introducing another state variable into the equilibrium
argument.

\textbf{Approximation to the ideal game.}
The quantities available to \system are structural approximations to the
ideal terms $D$ and $A_i$ in Eq.~\ref{eq:game_utilities}. Let
$\widehat{D}$ and $\widehat{A}_i$ denote the quantities induced by
$d_{\mathrm{mem}}$ and $d_{\mathrm{use}}$. Suppose their approximation
errors over the current feasible region satisfy
\begin{equation}
|D-\widehat{D}|\leq\varepsilon_D
\qquad
|A_i-\widehat{A}_i|\leq\varepsilon_A
\label{eq:approximation_error}
\end{equation}
The calibrated game therefore differs from the ideal zero-trust game only
through these measurement errors and the marginal memory-side cost
$\Delta c_i$. If the calibrated interaction is stable up to $\xi$ under
the measurable quantities, a unilateral deviation can improve the ideal
utility by at most the calibration error added on the two sides of the
comparison. In the standard approximate-equilibrium form,
\begin{equation}
u_j(s_j^*,s_{-j}^*)
\geq
u_j(s_j,s_{-j}^*)
-
\left(
\xi+2\varepsilon
\right)
\qquad
\varepsilon=\max\{\varepsilon_D,\varepsilon_A,\max_i\Delta c_i\}
\label{eq:approx_equilibrium}
\end{equation}
The argument follows directly from applying the uniform error bound once to
the current action and once to the deviating action. Hence the calibrated
state approaches the equilibrium of the ideal game as the structural
measurements become more accurate and the marginal auditing cost decreases.

This result also clarifies the role of the two constraints.
$d_{\mathrm{mem}}$ couples influence to the structure of the complete
memory state. $d_{\mathrm{use}}$ couples influence to the portions of memory
that other agents are currently relying on. Either source can restrict the
update, so a proposal receives full influence only when both remain within
their normal operating ranges.
\section{Calibration Details}
\label{app:calibration_details}

\subsection{Embedding-Based Memory}
\label{app:embedding_details}

Let
$\mathcal{M}_t=\{(\mathbf{e}_j,m_j)\}_{j=1}^{N}$ and
$\Delta_t=\{(\mathbf{e}_k^\Delta,m_k^\Delta)\}_{k=1}^{K}$.
We compute the two deviations in Eq.~\ref{eq:calibration} before the
candidate becomes active.

\textbf{Memory deviation.}
For each candidate embedding $\mathbf{e}_k^\Delta$, let $\hat r_k$ be the
number of existing memory vectors lying within radius $\gamma_t$. The radius
is the running median similarity radius observed by the base retrieval
index. Let $\bar r_t$ denote the running mean neighborhood size of existing
entries. We measure only excess concentration:
\begin{equation}
d_{\mathrm{mem}}^{\mathrm{emb}}
=
\frac{1}{K}
\sum_{k=1}^{K}
\frac{
[\hat r_k-\bar r_t]_+
}{
\bar r_t+\varepsilon
}
\label{eq:emb_memory_deviation}
\end{equation}
A candidate whose neighborhood is comparable to ordinary memory entries has
little deviation. A candidate located in an unusually dense region has
larger deviation because it can participate in many nearby retrievals.

The positive-part operator avoids penalizing sparse additions. Such an
addition may contain genuinely new information, while unusual concentration
is more directly related to excessive retrieval influence.

\textbf{Memory-use deviation.}
For every auditor with an eligible retrieval query $q_i$, let
\begin{equation}
L_i
=
\operatorname{top}\text{-}k(\mathcal{M}_t,q_i)
\qquad
L_i^\Delta
=
\operatorname{top}\text{-}k(\mathcal{M}_t\cup\Delta_t,q_i)
\label{eq:emb_retrieval_lists}
\end{equation}
The first list is already cached from normal retrieval. The second is
evaluated while the candidate remains staged. We compute
\begin{equation}
d_{\mathrm{use}}^{\mathrm{emb}}
=
\operatorname{median}_{i\in\mathcal{O}_t}
d_{\mathrm{RBO}}(L_i,L_i^\Delta)
\label{eq:emb_use_deviation}
\end{equation}
where $\mathcal{O}_t$ contains auditors with an eligible pre-update query.
Rank-biased overlap gives greater weight to changes near the top of the
retrieval list, which is appropriate because those entries exert the
largest downstream effect.

When no eligible auditor query is available, the use constraint is omitted
for that round and calibration is determined by the memory constraint.
The update can be recalibrated after additional interactions become
available.

\textbf{Applying calibrated influence.}
Let $s_0(q,m_j)$ denote the native retrieval value used by the underlying
memory architecture. For an entry carrying influence $\tau_j$, \system uses
\begin{equation}
s(q,m_j)=\tau_j\,s_0(q,m_j)
\label{eq:emb_weighted_retrieval}
\end{equation}
This form is independent of the particular similarity function used by the
base memory. In particular, it does not rely on rescaling an embedding
vector, which would leave cosine similarity unchanged after normalization.

\subsection{Graph-Based Memory}
\label{app:graph_details}

Let
$\mathcal{M}_t=(\mathcal{V}_t,\mathcal{E}_t,\ell)$ and
$\Delta_t=(\mathcal{V}_t^\Delta,\mathcal{E}_t^\Delta)$.
Candidate nodes and edges are first placed in a staging area and remain
invisible to agents during calibration.

\textbf{Relation compatibility.}
For two relation types $r$ and $r'$, let $c_t(r,r')$ be the number of nodes
whose local neighborhoods contain both relations. We normalize this count
by their marginal frequencies
\begin{equation}
\operatorname{cooc}_t(r,r')
=
\frac{
c_t(r,r')|\mathcal{V}_t|
}{
(c_t(r)+\varepsilon)(c_t(r')+\varepsilon)
}
\label{eq:relation_cooccurrence}
\end{equation}
Relations whose co-occurrence exceeds the running median are included in
$\mathcal{R}_{\mathrm{compat}}(r)$. This set is learned entirely from the
current graph and is updated as new edges become active.

\textbf{Memory deviation.}
For a candidate edge $e=(u,r,v)$, we first measure local support
\begin{equation}
s_{\mathrm{local}}(e)
=
\frac{
\left|
\{
(u,r',v')\in\mathcal{E}_t:
r'\in\mathcal{R}_{\mathrm{compat}}(r)
\}
\right|
}{
|\mathcal{N}(u)|+\varepsilon
}
\label{eq:local_support}
\end{equation}
We also inspect existing paths between $u$ and $v$. Let
$\mathcal{P}_{u,v}$ contain the paths found within the normal traversal
depth of the graph memory. Their relation support is
\begin{equation}
s_{\mathrm{path}}(e)
=
\frac{1}{|\mathcal{P}_{u,v}|}
\sum_{\pi\in\mathcal{P}_{u,v}}
\mathbf{1}
\left[
\operatorname{type}(\pi)
\cap
\mathcal{R}_{\mathrm{compat}}(r)
\neq\emptyset
\right]
\label{eq:path_support}
\end{equation}
When no path exists, we use $s_{\mathrm{local}}(e)$ alone. Otherwise we set
$\kappa_t(e)=
\frac{1}{2}\left(s_{\mathrm{local}}(e)+s_{\mathrm{path}}(e)\right)$.
The update-level memory deviation is
\begin{equation}
d_{\mathrm{mem}}^{\mathrm{graph}}
=
\frac{1}{|\mathcal{E}_t^\Delta|}
\sum_{e\in\mathcal{E}_t^\Delta}
\left(
1-\kappa_t(e)
\right)
\label{eq:graph_memory_deviation}
\end{equation}
Thus edges supported by existing relation patterns produce little change,
while structurally isolated edges consume more of the permitted influence
budget.

\textbf{Memory-use deviation.}
Let $\pi_i$ be the most recent eligible path traversed by auditor $a_i$
before the candidate entered staging. Define
$\operatorname{connect}(e,\pi_i)=1$ when $e$ touches a node on $\pi_i$ or
connects to that path through a compatible relation. We compute
\begin{equation}
d_{\mathrm{use}}^{\mathrm{graph}}
=
\operatorname{median}_{i\in\mathcal{O}_t}
\left[
1-
\frac{1}{|\mathcal{E}_t^\Delta|}
\sum_{e\in\mathcal{E}_t^\Delta}
\operatorname{connect}(e,\pi_i)
\right]
\label{eq:graph_use_deviation}
\end{equation}
This quantity is small when the proposed update extends regions of the graph
that agents are currently using. It becomes larger when the update forms a
largely disconnected structure.

\textbf{Applying calibrated influence.}
Every active edge inherits the influence $\tau_t$ of the update that created
it. If the underlying graph memory assigns native strength
$s_0(\pi)$ to a retrieved path $\pi$, \system uses
\begin{equation}
s(\pi)
=
s_0(\pi)
\prod_{e\in\pi}\tau_e
\label{eq:graph_weighted_retrieval}
\end{equation}
A path containing a weakly supported edge therefore receives proportionally
less influence during later reasoning.

\subsection{Adaptive Calibration Ranges}
\label{app:adaptive_ranges}

The tolerances $\epsilon_t$ and $\eta_t$ in
Eq.~\ref{eq:calibration} describe the amount of memory change that is
ordinary for the current system. We estimate them from recent proposals,
which avoids task-specific threshold tuning.

Let $\mathcal{H}_t$ be a bounded history window containing the raw
$d_{\mathrm{mem}}$ and $d_{\mathrm{use}}$ values observed before influence
calibration. We use
\begin{equation}
\epsilon_t
=
\operatorname{median}
\left\{
d_{\mathrm{mem}}^{(\tau)}:
\tau\in\mathcal{H}_t
\right\}
\qquad
\eta_t
=
\operatorname{median}
\left\{
d_{\mathrm{use}}^{(\tau)}:
\tau\in\mathcal{H}_t
\right\}
\label{eq:adaptive_tolerances}
\end{equation}
Using raw pre-calibration deviations avoids a feedback loop in which heavily
downweighted updates artificially reduce future tolerances.

During cold start, each statistic is the median of the observations
available so far. The first observation therefore defines the initial
operating range. As more rounds accumulate, individual early observations
quickly lose influence on the median.

Because both constraints in Eq.~\ref{eq:calibration} are linear in
$\tau$, the optimization has a direct solution whenever the corresponding
deviations are positive
\begin{equation}
\tau_t
=
\min
\left\{
1,\
\frac{\epsilon_t}{d_{\mathrm{mem}}+\varepsilon},\
\frac{\eta_t}{d_{\mathrm{use}}+\varepsilon}
\right\}
\label{eq:closed_form_tau}
\end{equation}
When no auditor interaction is available, the third term is omitted. This
closed form means that calibration requires no iterative optimization at
inference time.

\subsection{Discussion: How \system Handles the Motivating Cases}
\label{app:motivating_cases}

The three cases in Figure~\ref{fig:intro} differ in which agents appear confident, but they share the same underlying problem: the decision about shared memory is too easily determined by the agents' expressed judgments. \system does not treat these cases separately. Instead, it applies the same procedure to every proposed update: the update is first evaluated against the existing memory and against the memory already being used by other agents, and its subsequent influence is calibrated from these two checks. Agent confidence, agreement, or disagreement does not directly determine this calibration.

\textbf{Confident contributor, uncertain auditors.}
When an incorrect contributor is confident and the other agents fail to challenge it, \system does not interpret their lack of objection as evidence that the update is reliable. If the proposed update conflicts with the accumulated memory or substantially changes information that auditors have already retrieved or traversed, its influence is reduced even when no auditor explicitly raises the conflict. Thus, a contributor cannot gain greater memory influence simply by sounding more certain than the other agents.

\textbf{Confident auditors, tentative contributor.}
The converse case is handled in the same way. Auditors do not have a direct veto over a proposed update, and their stated confidence is not used as a calibration signal. A tentatively expressed update can therefore retain substantial influence when it agrees with the existing memory and does not disrupt the evidence already being used by other agents. The mechanism depends on the update's relation to memory, rather than on which side of the debate sounds more confident.

\textbf{All agents are uncertain.}
Likewise, uncertainty among all participants need not force a proposed update to be discarded. Once an update is proposed, \system can keep it in memory with an influence determined by the same two checks. An update that is well supported can remain influential despite uncertain language, while a poorly supported one is retained with reduced influence. This separates whether information is stored from how strongly it affects later retrieval. Of course, \system cannot recover useful information that no agent proposes in the first place; its role is to calibrate proposed memory updates rather than generate missing ones.

Taken together, these cases illustrate why \system operates below the level of debate confidence or voting. The same calibration rule applies regardless of whether confidence lies with the contributor, with the auditors, or with neither side. What changes across cases is the behavior of the agents; what remains fixed is how \system determines the influence of the resulting memory update.

\section{Inference-Time Efficiency}
\label{app:efficiency_details}

\system adds memory-side computation while leaving the agent reasoning
process unchanged. This distinction is important because language-model
generation dominates the token and latency cost of the multi-agent system.

\textbf{Embedding memory.}
The neighborhood statistic in
Eq.~\ref{eq:emb_memory_deviation} is obtained from the nearest-neighbor index
already maintained by the memory architecture. No second index is required.

For the memory-use constraint, the pre-update list $L_i$ has already been
returned to auditor $a_i$ during normal reasoning. We cache this result.
To construct $L_i^\Delta$, the implementation evaluates the small set of
candidate embeddings against $q_i$ and merges them with the cached top-$k$
entries. A complete second search over $\mathcal{M}_t$ is therefore
unnecessary.

The influence value $\tau_j$ is stored as one scalar with each entry.
During retrieval, it is applied to the native retrieval value after the
candidate set has been obtained. The underlying embedding representation
and nearest-neighbor data structure remain unchanged.

\textbf{Graph memory.}
The local relation check reads adjacency information already stored by the
graph memory. Relation co-occurrence counts are updated incrementally after
an edge becomes active. Existing paths used for
Eq.~\ref{eq:path_support} can be taken from the graph retrieval process when
available; a bounded search is performed only when the relevant endpoint
pair has no cached path information.

Auditor paths used in Eq.~\ref{eq:graph_use_deviation} are recorded from
normal traversals before the update is staged. Candidate edges remain in a
temporary buffer during calibration, so this procedure does not require a
second visible graph.

The influence value of each committed edge is stored as an edge attribute.
Path retrieval accumulates these values together with the native path
strength. No auxiliary graph is maintained.

\textbf{No additional language-model generation.}
The calibration procedure never asks an agent to judge an update, explain a
decision, report confidence, or perform a separate verification turn.
Consequently, \system introduces no calibration prompt and no additional
generated tokens. The added work consists of memory-index lookup, cached
result comparison, graph lookup, and scalar arithmetic.

\textbf{Persistent state.}
For embedding memory, the main additional state is one $\tau$ value per
stored entry and a bounded history of recent deviation measurements. For
graph memory, each edge carries the same scalar value, together with the
relation co-occurrence table used by
$\mathcal{R}_{\mathrm{compat}}$. These quantities are updated online with
the underlying memory.

\section{Experiments Results Analysis}
\label{app:expt}

\subsection{Configuration}
\label{app:configuration}

\textbf{LLM backbone and implementation.}
All experiments in Table~\ref{tab:main} use
Qwen3-VL-8B-Instruct~\citep{qwen3technicalreport} as the
agent backbone, deployed locally via Ollama with
$N = 6$ agents per system.  We select Qwen3-VL-8B-Instruct because it achieves SOTA on AgentBench and ToolBench among 8B models under the Apache 2.0 license, while its native function-calling interface simplifies multi-agent orchestration.  Alternative
backbones (DeepSeek-R1-Distill-7B, Llama-3.1-8B)
are evaluated in Appendix~\ref{app:diff-backbone}.

Each agent operates independently with no shared
parameters or centralised coordination, except for
interaction through the shared memory $\mathcal{M}$.
\system's calibration is implemented as a standalone
Python process with no generative LLM dependency at
runtime. It computes the memory deviation
$d_{\mathrm{mem}}$, use deviation $d_{\mathrm{use}}$,
and calibrated influence $\tau_t$ from observable
memory state and existing agent interactions, consistent
with the zero-trust setting in Definition~\ref{def:zero_trust}.

Embedding-space computations for
$d_{\mathrm{mem}}^{\mathrm{emb}}$ and
$d_{\mathrm{use}}^{\mathrm{emb}}$ use
all-MiniLM-L6-v2~\citep{wang2020minilm} with
approximate nearest-neighbour retrieval for efficient
scaling to large memory sizes.  Experiments with
varying agent counts $N \in \{4, 10\}$ are reported
in Appendix~\ref{app:agent_scaling}.

\medskip
\textbf{Hyperparameter settings.}
We summarise the key hyperparameters in
Table~\ref{tab:hyperparams}.  Unlike prior
multi-agent or memory-based systems that rely on
manually tuned thresholds, \system\ minimises the
number of fixed hyperparameters by estimating most
quantities directly from data.

\begin{table}[h]
\centering
\small
\caption{Key hyperparameters and configurations
  in \system.}
\label{tab:hyperparams}
\begin{tabular}{l|l|l}
\toprule
\textbf{Component} & \textbf{Parameter}
  & \textbf{Value / Strategy} \\
\midrule
Agents & Number of agents $N$
  & 6 (default), varied in Appendix \\
LLM & Backbone
  & Qwen3-VL-8B-Instruct \\
Memory & Top-$k$ retrieval
  & $k = 5$ \\
Embedding & Dimension $d$
  & 384 (MiniLM default) \\
\midrule
Calibration & Operating ranges $\epsilon_t,\eta_t$
  & Running medians of raw deviations \\
Embedding & Similarity radius
  & Median similarity (adaptive) \\
Graph & Path depth $L_{\max}$
  & $\lceil \log |\mathcal{V}| \rceil$ \\
Graph & Relation compatibility
  & Online co-occurrence statistics \\
Use deviation & Auditor aggregation
  & Median over eligible interactions \\
\bottomrule
\end{tabular}
\end{table}

\medskip
\textbf{Adaptive parameter estimation.}
A key design principle of \system\ is that most
hyperparameters are \emph{data-adaptive} rather than
manually tuned:

\begin{itemize}
    \item The operating ranges $\epsilon_t$ and
    $\eta_t$ are running medians of recent raw
    $d_{\mathrm{mem}}$ and $d_{\mathrm{use}}$ values,
    eliminating task-specific tuning of the influence
    calibration.

    \item The embedding similarity radius is derived
    from the running median of observed similarities,
    while graph relation compatibility is estimated
    from historical relation co-occurrence patterns.

    \item The use deviation aggregates eligible
    auditor retrievals or traversals with a median,
    requiring no auditor-specific reliability
    parameter.

    \item Graph traversal depth $L_{\max}$ scales with
    the logarithm of graph size, reflecting typical
    structural properties of real-world knowledge
    graphs while maintaining computational efficiency.
\end{itemize}

This adaptive design serves two purposes.  First, it
reduces the risk of overfitting hyperparameters to
specific benchmarks, improving generalisation across
tasks and memory architectures.  Second, it aligns
with the zero-trust assumption by avoiding reliance
on any externally tuned or agent-provided signals.
All thresholds are grounded in observable system
statistics, ensuring that calibration remains
consistent and robust under adversarial or
non-cooperative agent behaviours.

\subsection{Dataset Statistics}
\label{app:datasets}

\noindent\textbf{Why these benchmarks.} The four benchmarks are
selected to stress-test the primary failure modes identified by
our zero-trust game analysis: (i) a single dominant contributor
injecting a locally plausible but globally inconsistent entry, and
(ii) auditors rubber-stamping an incorrect proposal when 
uncertain or sycophantic. These failures cause particular harm
in long-horizon tasks-premature goal states in embodied action,
fabricated intermediate facts in multi-hop reasoning, and temporal
contradictions in lifelong conversation. The selected datasets
span knowledge-intensive and action-intensive domains, providing
the diverse challenge needed to evaluate \system's calibration.

Table~\ref{tab:dataset_stats} summarises the four benchmarks
used in our experiments. We report the number of task instances,
the average interaction length, the evaluation metric, and the 
domain of each benchmark.

\begin{table}[h]
\centering
\small
\caption{Benchmark statistics for the four evaluation
  benchmarks.}
\label{tab:dataset_stats}
\begin{tabular}{l|ccccc}
\toprule
\textbf{Benchmark} & \textbf{\#Tasks}
  & \textbf{Avg.\ Len.} & \textbf{Metric}
  & \textbf{Domain} & \textbf{Source} \\
\midrule
ALFWorld  & 134 & $\sim$25 steps & Success rate
  & Embodied  & \cite{shridhar2020alfworld} \\
Webshop  & 1000 & $\sim$30 steps & Reward
  & Embodied  & \cite{yao2022webshop} \\
HotpotQA  & 500 & 1 query  & Accuracy
  & Knowledge & \cite{yang2018hotpotqa} \\
LoCoMo     & 500 & $<$300 turns  & F1
  & Knowledge & \cite{maharana2024evaluating} \\
\bottomrule
\end{tabular}
\end{table}

\textbf{ALFWorld} is a text-based embodied environment
derived from ALFRED, covering six task categories:
\texttt{pick\_and\_place}, \texttt{look\_at\_obj\_in\_light},
\texttt{pick\_clean\_then\_place},
\texttt{pick\_heat\_then\_place},
\texttt{pick\_cool\_then\_place}, and
\texttt{pick\_two\_obj\_and\_place}. We use the
\texttt{valid\_unseen} split (134 tasks), where each task
requires a sequence of physical actions with an average length of
approximately 25 steps. \textbf{Why it matters:} The long action
chain and strict precondition dependencies make this benchmark
particularly susceptible to premature goal-state injection-a
corrupted memory entry that incorrectly asserts a task is complete
can terminate the entire episode with zero reward (see Case~2 in
Appendix~\ref{app:case_study}).

\textbf{WebShop} is an embodied decision-making benchmark
that simulates online shopping through natural language
interactions. Each task requires an agent to navigate search
results, apply filters, and inspect product pages to locate a
specified item. We evaluate on 1,000 randomly sampled tasks, with
an average trajectory length of approximately 30 steps.
\textbf{Why it matters:} The large, dynamic search space forces
frequent memory writes under uncertainty; a dominant agent that
confidently commits an incorrect product attribute (e.g., a wrong
color or price) can mislead the entire downstream search process.

\textbf{HotpotQA} is a multi-hop question answering benchmark
constructed from Wikipedia. Each instance requires reasoning
across multiple documents to derive the final answer. We use 500
questions drawn from the full HotpotQA dev set, each consisting
of a single query, and report accuracy. \textbf{Why it matters:}
A hallucinated intermediate fact committed to shared memory can 
override correct evidence later retrieved by other agents, 
directly testing whether $d_{\mathrm{use}}$ captures
inconsistent overwrites in the retrieval pool (see Case~1).

\textbf{LoCoMo} is a long-context conversational memory
benchmark. Each original dialogue spans up to 35 sessions and
averages 300 turns, simulating a lifelong, multi-session human
conversation. From these long-form dialogues, we sample 500 
question-answering tasks, each derived from a \emph{separate} 
dialogue to guarantee statistical independence. Every task 
presents the full dialogue history in memory and asks one 
evaluative question requiring cross-session temporal reasoning. 
We report F1 as the primary evaluation metric. \textbf{Why it 
matters:} The extreme temporal span (300 turns) naturally creates 
memory decay, contradiction, and topic drift across sessions.
This tests the most dangerous failure mode for our theory: a
contributor who proposes a confident but incorrect summary, and a
room of auditors who, due to the cost of searching 300-turn
memory, prefer to agree rather than verify (see Case~3).

\subsection{Extended Analysis of Overall Effectiveness}
\label{app:main_discussion}

This section expands on our insights of \S\ref{ssec:main-expt} with deeper mechanistic explanation.

\textbf{Why isolated checks saturate?} 
On embedding memory, the gain from Vanilla to the strongest per-entry baseline (A-MemGuard) is modest, and \system\ then adds a further substantial improvement on top, exceeding the combined gain of all three intermediate baselines (LLM Audit, PPL Filter, A-MemGuard). The three baselines also cluster tightly with frequent rank inversions, suggesting they capture overlapping signals rather than complementary ones.

Both observations are consistent with the saturation hypothesis: isolated entry level checks evaluate only the \emph{intrinsic quality} of $\Delta_t$ (grammaticality, fluency, perplexity), while the dominant errors in a growing memory $\mathcal{M}_t$ arise from \emph{relational inconsistency}, where an entry reads perfectly in isolation yet contradicts existing memory. \system breaks this ceiling by constraining influence through both $d_{\mathrm{mem}}$, which measures deviation from the current memory structure, and $d_{\mathrm{use}}$, which measures disruption to memory already used by other agents.

\textbf{Qualitative pattern of ceiling aware gains.} 
An inverse scaling effect is visible across all configurations: \system's gain over the next best baseline is largest where that baseline is weakest, and shrinks where the baseline already performs well. This pattern holds consistently across reasoning intensive and action intensive benchmarks alike, with the largest improvements appearing in settings where existing safeguards struggle the most.

This shrinkage follows from continuous influence calibration: when the incoming updates already remain close to the normal memory and use ranges, $\tau_t$ stays near one and \system makes only a small change. Importantly, this inverse scaling pattern is \emph{inconsistent} with the hypothesis that performance gains stem from external knowledge leakage, which would produce a constant additive boost regardless of baseline strength. Instead, it is \emph{consistent} with \system limiting the influence of updates that conflict with the accumulated memory state.

\textbf{Memory benchmark interaction.} 
The strongest results for reasoning intensive benchmarks come from AriGraph paired with \system, while the strongest action task results come from G-Memory paired with \system. This complementarity is already visible in the Vanilla setting: AriGraph leads on multi hop QA, while G-Memory leads on embodied action and long dialogue tasks. These pre existing biases reflect each architecture's native design. AriGraph's semantic knowledge graph with episodic memory is well suited to multi hop QA, while G-Memory's three tier hierarchical structure captures the long action chains and multi session dialogue history of action and conversational tasks.

\system's calibration \emph{inherits and amplifies} these biases rather than imposing a generic operation that flattens architectural differences. The deviations $d_{\mathrm{mem}}$ and $d_{\mathrm{use}}$ are instantiated from each architecture's geometry (embedding neighbourhoods and retrieval changes) or topology (graph neighbourhoods, paths, and traversals), so noisy updates receive less downstream influence while the architecture's native signal is preserved. The gap between AriGraph and G-Memory therefore widens under \system\ in the directions each architecture is naturally suited to, rather than narrowing toward a common baseline.

\textbf{Cross MAS consistency.} 
Absolute performance differs substantially across MAD frameworks, with AutoGen generally leading, followed by DyLAN and MacNet. Yet \system's gain over the next best baseline is comparably stable across all three frameworks. The same ranking and the same inverse scaling pattern persist regardless of which orchestration scheme is used.

The stability of the gain across orchestration regimes confirms that \system acts at the memory layer rather than at the coordination layer. Weak orchestration and strong orchestration receive similar lifts because calibration assigns $\tau_t$ when $\Delta_t$ enters memory, upstream of any MAS-specific aggregation. This decoupling is a practical advantage: \system\ can be deployed with any MAD framework without re-tuning or framework-specific adaptation.

\subsection{Ablation Study}
\label{app:ablation}
\begin{table*}[t]
\centering
\caption{Ablation study of \system's calibration components
  (\%). We ablate one component at a time, following the notation
  in Section~\ref{sec:calibration}. For embedding memory, we remove
  the memory deviation $d_{\mathrm{mem}}$ or memory-use deviation
  $d_{\mathrm{use}}$ directly. For graph memory, we additionally
  ablate the local and path consistency terms
  $c_{\mathrm{local}}$ and $c_{\mathrm{path}}$ that constitute
  $d_{\mathrm{mem}}$. We also remove calibrated influence to test
  the effect of applying $\tau_t$ during subsequent retrieval.
  All configurations use AutoGen + \texttt{Qwen3-VL-8B-Instruct},
  $N=6$ agents, in the benign setting (mean$\pm$std over 5 runs).}
\label{tab:ablation}
\setlength{\tabcolsep}{3pt}
\renewcommand{\arraystretch}{1.15}
\resizebox{\textwidth}{!}{
\begin{tabular}{l|l|cccc|cccc|cccc}
\toprule
& & \multicolumn{4}{c|}{\textbf{AutoGen}}
  & \multicolumn{4}{c|}{\textbf{MacNet}}
  & \multicolumn{4}{c}{\textbf{DyLAN}} \\
\cmidrule(lr){3-6} \cmidrule(lr){7-10} \cmidrule(lr){11-14}
\textbf{Memory} & \textbf{Ablation}
  & \textbf{HQA} & \textbf{LCM}
  & \textbf{ALF} & \textbf{WS}
  & \textbf{HQA} & \textbf{LCM}
  & \textbf{ALF} & \textbf{WS}
  & \textbf{HQA} & \textbf{LCM}
  & \textbf{ALF} & \textbf{WS} \\
\midrule

\multirow{3}{*}{\shortstack[l]{MemBank\\\scriptsize{(emb.)}}}

& w/o $d_{\mathrm{mem}}$
  & 46.8\stdval{6.9} & 33.8\stdval{6.1} & 74.0\stdval{3.6} & 33.9\stdval{7.9}
  & 44.6\stdval{9.2} & 30.4\stdval{10.0} & 56.2\stdval{1.0} & 45.0\stdval{6.6}
  & 42.0\stdval{0.9} & 32.9\stdval{5.8} & 58.3\stdval{2.8} & 36.7\stdval{3.5} \\

& w/o $d_{\mathrm{use}}$
  & 38.4\stdval{5.3} & 23.7\stdval{6.9} & 67.9\stdval{4.3} & 35.1\stdval{7.5}
  & 48.8\stdval{7.5} & 21.6\stdval{10.2} & 59.9\stdval{1.1} & 46.7\stdval{7.7}
  & 39.1\stdval{0.8} & 26.1\stdval{7.1} & 65.5\stdval{2.6} & 43.4\stdval{3.8} \\

& w/o calibrated influence
  & 49.7\stdval{6.3} & 34.9\stdval{5.5} & 71.6\stdval{3.7} & 40.1\stdval{7.7}
  & 51.1\stdval{8.0} & 32.3\stdval{11.4} & 65.2\stdval{1.3} & 44.8\stdval{5.8}
  & 53.6\stdval{0.8} & 36.7\stdval{7.0} & 65.6\stdval{2.6} & 47.5\stdval{3.5} \\

\midrule

\multirow{4}{*}{\shortstack[l]{G-Memory\\\scriptsize{(graph)}}}

& w/o $c_{\mathrm{local}}$ in $d_{\mathrm{mem}}$
  & 42.9\stdval{10.5} & 57.9\stdval{3.7} & 77.0\stdval{3.7} & 53.3\stdval{6.2}
  & 42.6\stdval{8.0} & 52.9\stdval{1.6} & 66.8\stdval{9.3} & 55.8\stdval{4.0}
  & 45.9\stdval{8.2} & 45.8\stdval{2.8} & 59.9\stdval{11.4} & 46.1\stdval{6.4} \\

& w/o $c_{\mathrm{path}}$ in $d_{\mathrm{mem}}$
  & 36.5\stdval{10.1} & 46.8\stdval{3.9} & 74.6\stdval{3.7} & 47.6\stdval{6.3}
  & 33.2\stdval{6.2} & 46.3\stdval{1.8} & 75.8\stdval{8.7} & 58.1\stdval{3.3}
  & 35.0\stdval{8.3} & 39.0\stdval{2.8} & 71.4\stdval{11.5} & 46.9\stdval{6.7} \\

& w/o $d_{\mathrm{use}}$
  & 39.2\stdval{10.4} & 51.2\stdval{3.2} & 77.7\stdval{3.6} & 54.3\stdval{6.8}
  & 39.3\stdval{6.5} & 54.8\stdval{1.5} & 62.7\stdval{9.8} & 50.4\stdval{3.4}
  & 36.5\stdval{8.3} & 41.2\stdval{2.4} & 61.4\stdval{11.9} & 45.8\stdval{5.6} \\

& w/o calibrated influence
  & 50.9\stdval{10.0} & 56.3\stdval{3.2} & 80.5\stdval{4.1} & 50.2\stdval{6.2}
  & 42.3\stdval{6.8} & 52.2\stdval{1.7} & 67.2\stdval{10.4} & 50.6\stdval{3.1}
  & 49.6\stdval{9.3} & 51.0\stdval{2.7} & 68.6\stdval{10.7} & 52.3\stdval{6.6} \\

\bottomrule
\end{tabular}}
\end{table*}

This part dissects \system's calibration into its main measurements to identify which elements drive the gains observed in RQ1. We ablate one component at a time on two representative memory architectures (MemBank and G-Memory) across all three MAS frameworks and four benchmarks. Table~\ref{tab:ablation} reports the results.

\textbf{Both memory-state and memory-use deviations contribute materially.}
Removing either deviation consistently reduces performance, showing that the two constraints capture complementary aspects of update quality. On embedding memory, removing $d_{\mathrm{mem}}$ eliminates the constraint that detects unusually concentrated additions to the existing representation, while removing $d_{\mathrm{use}}$ eliminates the signal measuring disruption to retrievals already used by other agents. The latter produces particularly large drops on several reasoning-intensive configurations, indicating that preserving the memory context actively supporting the debate is important beyond assessing the candidate against the memory state alone.

\textbf{Graph calibration benefits from both local and path-level structural evidence.}
For graph memory, ablating either $c_{\mathrm{local}}$ or $c_{\mathrm{path}}$ weakens the construction of $d_{\mathrm{mem}}$ and causes substantial degradation across benchmarks. The two terms provide distinct structural evidence: $c_{\mathrm{local}}$ measures whether the proposed relation is compatible with the neighborhoods around its endpoints, whereas $c_{\mathrm{path}}$ measures whether existing multi-hop structure supports the proposed relation. Removing $d_{\mathrm{use}}$, which measures compatibility with graph regions already traversed by auditors, produces additional losses. Together, these results support the shared formulation in Section~\ref{sec:calibration}: calibration benefits from considering both how an update changes the memory itself and how it interacts with memory already supporting other agents' reasoning.

\textbf{Calibrated influence is necessary to convert deviation measurements into downstream protection.}
Removing calibrated influence while retaining the underlying deviation measurements also reduces performance across most configurations. This ablation leaves the system able to measure $d_{\mathrm{mem}}$ and $d_{\mathrm{use}}$ but prevents those signals from regulating subsequent retrieval through $\tau_t$. The result confirms that the benefit of \system does not come only from detecting atypical updates; it depends on carrying the resulting calibration forward so that weakly supported entries exert proportionally less influence on later reasoning.

\textbf{Auditor-side aggregation has the smallest benign effect.}
The ablation that removes the reweighting mechanism causes the smallest drops in most benign configurations. This is expected because benign retrievals and traversals usually provide compatible evidence, so the robust aggregation differs little from a non-robust aggregate. Its value becomes more visible under Audit Collusion (Section~\ref{sec:rq2}), where a subset of adversarial interactions can be deliberately steered.

\textbf{No single component is sufficient.}
Even the most impactful individual ablation does not collapse performance to the Vanilla level. The ablated system always retains a clear margin above the corresponding Vanilla score, showing that detection, alignment, and reweighting provide complementary protection. The full method performs best when all calibration signals jointly constrain the calibrated influence $\tau_t$.

\subsection{Alternative LLM Backbones}
\label{app:diff-backbone}

Table~\ref{tab:main} uses Qwen3-VL-8B-Instruct as the
agent backbone throughout.  To verify that \system's
advantage is not specific to one LLM, we replicate
the MacNet evaluation on MemBank and G-Memory with
two additional 7B/8B-class open-source models:
\textbf{DeepSeek-R1-Distill-7B}~\citep{guo2025deepseek},
a reasoning-optimised distillation of DeepSeek-R1,
and \textbf{Llama-3.1-8B-Instruct}~\citep{patterson2022carbon},
a general-purpose instruction-tuned model.
Table~\ref{tab:backbone} reports the results.

\begin{table*}[t]
\centering
\caption{Performance with alternative LLM backbones (\%) on MacNet, $N=6$ agents, benign setting.}
\label{tab:backbone}
\setlength{\tabcolsep}{3pt}
\renewcommand{\arraystretch}{1.15}
\resizebox{\textwidth}{!}{
\begin{tabular}{l|l|cccc|cccc|cccc}
\toprule
& & \multicolumn{4}{c|}{\textbf{Qwen3-VL-8B-Instruct} (primary)}
  & \multicolumn{4}{c|}{\textbf{DeepSeek-R1-Distill-7B}}
  & \multicolumn{4}{c}{\textbf{Llama-3.1-8B}} \\
\cmidrule(lr){3-6} \cmidrule(lr){7-10} \cmidrule(lr){11-14}
\textbf{Memory} & \textbf{Safeguard}
  & \textbf{HQA} & \textbf{LCM}
  & \textbf{ALF} & \textbf{WS}
  & \textbf{HQA} & \textbf{LCM}
  & \textbf{ALF} & \textbf{WS}
  & \textbf{HQA} & \textbf{LCM}
  & \textbf{ALF} & \textbf{WS} \\
\midrule

\multirow{5}{*}{\shortstack[l]{MemBank\\\scriptsize{(emb.)}}}
& Vanilla
  & 34.0\stdval{7.9} & 21.1\stdval{11.0} & 49.1\stdval{5.2} & 21.3\stdval{8.3}
  & 35.4\stdval{7.9} & 25.1\stdval{10.0} & 45.5\stdval{4.6} & 25.1\stdval{5.1}
  & 31.0\stdval{8.9} & 17.6\stdval{7.7} & 46.6\stdval{2.1} & 22.0\stdval{7.2} \\
& + LLM Audit
  & 38.5\stdval{9.2} & 22.2\stdval{12.1} & 62.3\stdval{5.6} & 31.2\stdval{5.7}
  & 46.4\stdval{8.1} & 26.5\stdval{12.0} & 60.2\stdval{3.4} & 30.9\stdval{8.5}
  & 38.5\stdval{9.4} & 19.4\stdval{9.5} & 55.7\stdval{6.0} & 28.9\stdval{6.3} \\
& + PPL Filter
  & 39.8\stdval{10.2} & 27.9\stdval{10.5} & 64.0\stdval{1.7} & 25.8\stdval{9.3}
  & 35.5\stdval{7.8} & 34.0\stdval{12.9} & 55.4\stdval{7.9} & 23.2\stdval{7.1}
  & 34.8\stdval{10.9} & 21.3\stdval{12.7} & 56.3\stdval{3.1} & 24.7\stdval{7.9} \\
& + A-MemGuard
  & 40.7\stdval{7.1} & 33.9\stdval{7.6} & 56.5\stdval{3.7} & 43.4\stdval{6.2}
  & 45.7\stdval{5.9} & 34.0\stdval{7.4} & 48.8\stdval{4.8} & 37.3\stdval{6.9}
  & 46.2\stdval{7.8} & 27.6\stdval{9.9} & 51.7\stdval{3.8} & 32.9\stdval{7.9} \\
& \textbf{+ \system}
  & \textbf{62.3}\stdval{7.7} & \textbf{35.4}\stdval{10.4} & \textbf{71.3}\stdval{6.3} & \textbf{56.9}\stdval{5.8}
  & \textbf{46.9}\stdval{6.8} & \textbf{34.0}\stdval{11.1} & \textbf{60.2}\stdval{4.8} & \textbf{50.0}\stdval{7.1}
  & \textbf{46.2}\stdval{8.0} & \textbf{32.1}\stdval{7.7} & \textbf{56.3}\stdval{3.7} & \textbf{45.2}\stdval{7.4} \\

\midrule

\multirow{5}{*}{\shortstack[l]{G-Memory\\\scriptsize{(graph)}}}
& Vanilla
  & 34.8\stdval{8.6} & 50.1\stdval{1.6} & 71.3\stdval{13.0} & 48.8\stdval{5.3}
  & 42.6\stdval{6.2} & 53.2\stdval{1.7} & 63.6\stdval{13.2} & 38.8\stdval{6.2}
  & 29.3\stdval{6.8} & 43.3\stdval{1.4} & 61.5\stdval{13.2} & 40.0\stdval{4.4} \\
& + LLM Audit
  & 38.5\stdval{8.4} & 48.0\stdval{1.7} & 71.5\stdval{12.1} & 48.5\stdval{5.9}
  & 40.1\stdval{6.1} & 53.2\stdval{1.8} & 63.6\stdval{11.6} & 42.4\stdval{4.9}
  & 33.6\stdval{7.6} & 50.1\stdval{1.4} & 62.3\stdval{12.1} & 42.9\stdval{3.8} \\
& + PPL Filter
  & 39.4\stdval{5.5} & 52.9\stdval{1.1} & 68.0\stdval{10.2} & 59.5\stdval{3.5}
  & 46.7\stdval{7.3} & 52.0\stdval{1.3} & 63.6\stdval{11.5} & 46.6\stdval{3.1}
  & 39.9\stdval{7.2} & 45.7\stdval{1.7} & 60.3\stdval{10.4} & 48.7\stdval{4.1} \\
& + A-MemGuard
  & --- & --- & --- & ---
  & --- & --- & --- & ---
  & --- & --- & --- & --- \\
& \textbf{+ \system}
  & \textbf{50.7}\stdval{6.6} & \textbf{58.6}\stdval{1.7} & \textbf{72.3}\stdval{6.9} & \textbf{67.6}\stdval{2.7}
  & \textbf{52.9}\stdval{6.8} & \textbf{53.2}\stdval{1.4} & \textbf{63.6}\stdval{9.5} & \textbf{60.0}\stdval{3.9}
  & \textbf{44.7}\stdval{8.5} & \textbf{58.5}\stdval{1.3} & \textbf{60.3}\stdval{7.0} & \textbf{57.6}\stdval{3.1} \\

\bottomrule
\end{tabular}}
\end{table*}

\textbf{\system ranks first across all backbones.}
Across all memory, benchmark, and backbone combinations, \system\
achieves the highest score. The relative ordering among safeguards
is largely preserved, with \system consistently outperforming existing
safeguards, while intermediate baselines show occasional variations
across settings as observed in Table~\ref{tab:main}.

\textbf{Backbone strength shifts absolute levels
but not relative patterns.}
Qwen3-VL-8B-Instruct generally achieves higher absolute performance
among the three backbones, followed by DeepSeek-R1-Distill-7B and
Llama-3.1-8B. The gap is consistent across safeguards: weaker
backbones produce lower scores for most methods, not just for
\system. Crucially, \system's \emph{relative} gain over the best
baseline remains substantial across all backbones, confirming that
calibration provides consistent benefit regardless of backbone.

\textbf{Calibration remains effective across backbone capabilities.}
The interaction between backbone and calibration shows a consistent
pattern: stronger backbones produce higher-quality memory entries,
which in turn provide \system richer memory structure and memory-use
evidence for calibration. Across different backbones, \system
maintains strong performance improvements, confirming that it
amplifies backbone capability rather than substituting for it:
calibration is complementary to, not independent of, the underlying
LLM's capability.


\subsection{Effect of Agent Count}
\label{app:agent_scaling}

The main text reports results for $N = 6$ agents.
This appendix evaluates how \system's effectiveness
varies with the number of participating agents.
We test $N \in \{4, 6, 10\}$ on the AutoGen framework
with Qwen3-VL-8B-Instruct and the G-Memory backbone in the
benign setting, where $N=6$ corresponds to the default
configuration used in the main experiments.

\begin{table}[h]
\centering
\small
\setlength{\tabcolsep}{4pt}
\caption{Effect of agent count $N$ on benign-setting
performance (AutoGen, Qwen3-VL-8B-Instruct, G-Memory backbone).
Values are reported in percentage (mean$\pm$std over 5 runs).
$\Delta$ denotes the gain of \system\ over the next-best safeguard
under the same agent count.}
\label{tab:agent_scaling}
\renewcommand{\arraystretch}{1.15}
\begin{tabular}{l|l|cccc|c}
\toprule
\textbf{Benchmark} & \textbf{$N$}
& \textbf{Vanilla} & \textbf{LLM Audit}
& \textbf{PPL Filter} & \textbf{\system}
& $\boldsymbol{\Delta}$ \\
\midrule
\multirow{2}{*}{HotpotQA}
& 4 & 33.9\stdval{9.4} & 36.5\stdval{8.9}
& 42.1\stdval{8.4} & \textbf{50.5}\stdval{8.0} & $+8.4$ \\
& 10 & 36.8\stdval{8.8} & 39.4\stdval{8.3}
& 45.6\stdval{7.9} & \textbf{53.6}\stdval{7.5} & $+8.0$ \\
\midrule
\multirow{2}{*}{LoCoMo}
& 4 & 41.5\stdval{4.5} & 47.6\stdval{4.8}
& 51.0\stdval{4.5} & \textbf{57.1}\stdval{3.4} & $+6.1$ \\
& 10 & 44.1\stdval{3.9} & 50.7\stdval{4.3}
& 54.0\stdval{4.0} & \textbf{60.0}\stdval{2.8} & $+6.0$ \\
\midrule
\multirow{2}{*}{ALFWorld}
& 4 & 70.4\stdval{4.1} & 74.9\stdval{3.5}
& 76.8\stdval{3.3} & \textbf{82.4}\stdval{3.5} & $+5.6$ \\
& 10 & 73.0\stdval{3.6} & 77.4\stdval{3.0}
& 79.1\stdval{2.9} & \textbf{85.1}\stdval{3.0} & $+6.0$ \\
\midrule
\multirow{2}{*}{WebShop}
& 4 & 37.8\stdval{6.6} & 39.5\stdval{6.4}
& 46.1\stdval{6.3} & \textbf{52.1}\stdval{6.0} & $+6.0$ \\
& 10 & 40.4\stdval{6.1} & 42.0\stdval{5.8}
& 49.4\stdval{5.9} & \textbf{54.7}\stdval{5.5} & $+5.3$ \\
\bottomrule
\end{tabular}
\end{table}

\textbf{Small agent counts ($N = 4$).}
With only four agents, each update is evaluated using a
limited number of eligible auditor memory interactions.
Despite this reduced diversity,
\system consistently outperforms all baselines by a
substantial margin across benchmarks. Compared to the
next-best safeguard, the absolute gains remain stable
across benchmarks, indicating that
\system does not rely on large numbers of auditors to
function effectively. However, we observe slightly
higher variance at $N = 4$, because calibration signals
are estimated from fewer auditor interactions.

\textbf{Moderate agent counts ($N = 6$).}
At $N = 6$, performance improves modestly across all
methods, with the largest gains observed when moving
from smaller agent counts to the default setting. This
suggests that increasing the number of agents initially
provides more diverse reasoning traces and memory
interactions, improving both memory quality and
calibration accuracy. Notably,
\system achieves strong overall performance at this
setting, balancing interaction diversity and
coordination overhead.

\textbf{Larger agent counts ($N = 10$).}
Increasing the agent count further to $N = 10$ yields
only marginal improvements over $N = 6$. While
calibration benefits from more diverse auditor interactions,
these gains are offset by increased redundancy among
agent reasoning traces and memory accesses. As a result,
performance begins to saturate,
indicating diminishing returns from additional agents.
Importantly, \system maintains a consistent advantage
over all baselines, suggesting that its effectiveness
is largely insensitive to the exact choice of $N$.

\textbf{Scaling implications.}
Overall, agent count primarily affects the \emph{quality
of collective memory interactions} rather than acting as a dominant
performance driver. \system requires no assumption about
a minimum number of reliable agents
(Definition~\ref{def:zero_trust}) and remains effective
even at small $N$. Increasing $N$ improves stability
and slightly enhances performance, but the benefits
quickly plateau beyond $N = 6$. For very large agent
pools ($N > 50$), eligible auditor retrievals or traversals can be
sub-sampled to a random subset of $\min(N, 20)$ agents
with negligible impact on calibration quality, ensuring
scalability without increasing latency.

\subsection{Effect of Debate Rounds (RQ1 Supplement)}
\label{app:debate_rounds}

Figure~\ref{fig:debate_rounds} examines how
performance evolves as the number of debate rounds
increases from $1$ to $5$, using LoCoMo (F1) on
MacNet + Qwen3-VL-8B-Instruct with $N=6$ agents.

\begin{figure}[h]
\centering
\includegraphics[width=\linewidth]{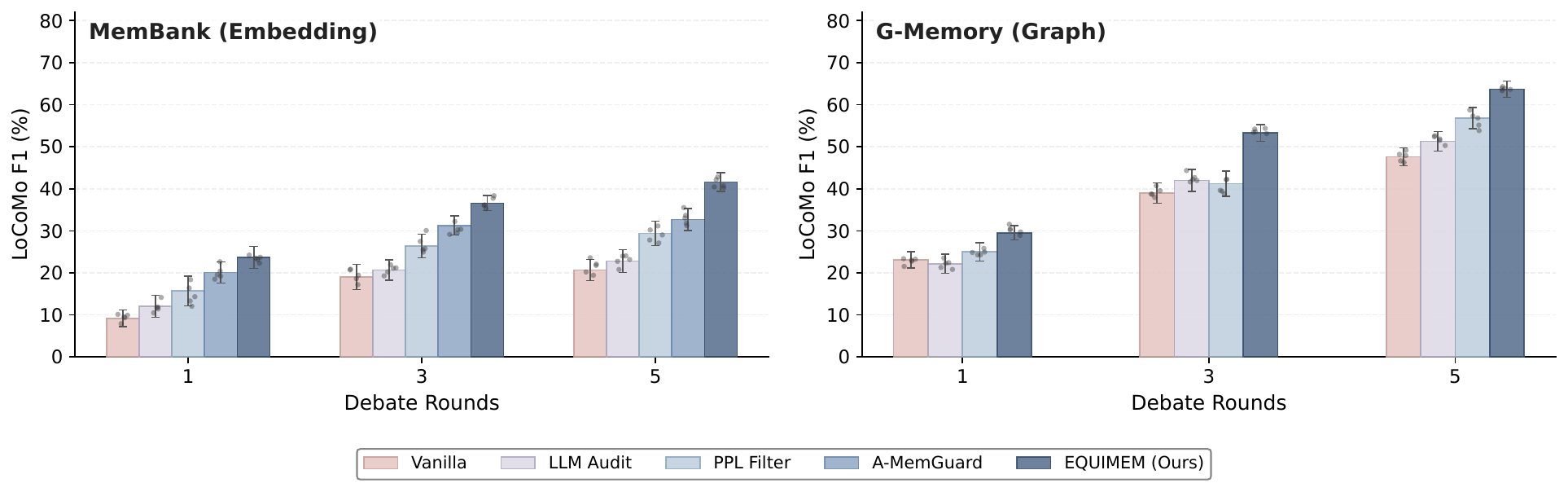}
\caption{LoCoMo F1 vs.\ debate rounds on two memory
architectures (MacNet, Qwen3-VL-8B-Instruct, $N=6$).}
\label{fig:debate_rounds}
\end{figure}

\textbf{Embedding memory saturates early.}
On MemBank (left panel), all methods show a sharp
rise from round~1 to round~3, followed by near-zero
improvement from round~3 to round~5. This saturation
occurs because embedding memory accumulates dense
vectors that increasingly overlap, reducing the marginal
information content of later rounds. \system\
follows the same saturation shape but remains at a
higher performance level throughout all rounds, with
the gap over Vanilla widening at later rounds. The
baselines (LLM~Audit, PPL~Filter, A-MemGuard) also
saturate, clustering between Vanilla and \system.

\textbf{Graph memory benefits from longer debate.}
On G-Memory (right panel), all methods improve
steadily across all debate rounds. Vanilla shows
consistent gains as the number of rounds increases,
while \system\ further amplifies this trend. G-Memory's
hierarchical knowledge graph continues to benefit from
additional debate rounds because each round's interactions
extend the graph's coverage into previously unreached
regions of the reasoning space. As a result, \system\
maintains an increasing advantage over Vanilla at later
rounds.

\textbf{Calibration benefit grows with rounds.}
In both panels, the gap between \system\ and
baselines tends to widen at later rounds. This is
expected: as more entries accumulate in
$\mathcal{M}_t$, the fraction of locally plausible
but globally inconsistent entries grows, and
\system's detection and alignment calibration becomes
increasingly valuable. Isolated entry-level
checks (LLM~Audit, PPL~Filter, A-MemGuard) cannot
leverage the growing memory state and therefore
plateau alongside Vanilla.

\subsection{Adversarial Baseline Comparison}
\label{app:adversarial_baselines}

This appendix supplements Section~\ref{sec:rq2} by
reporting the adversarial performance of LLM~Audit
and PPL~Filter under the same two attack schemes,
memory architectures, benchmarks, and adversary
counts used for \system\ in the main text.  All
experiments use MacNet + Qwen3-VL-8B-Instruct with $N=6$
agents. Algorithms~\ref{alg:memory_poisoning}
and~\ref{alg:audit_collusion} detail the two
adaptive attack procedures used throughout this
section and Section~\ref{sec:rq2}.

\begin{algorithm}[t]
\caption{Memory Poisoning Attack}
\label{alg:memory_poisoning}
\begin{algorithmic}[1]
\Require Adversary set $\mathcal{K} \subset \mathcal{N}$,
  memory state $\mathcal{M}_t$, target safeguard type
\Ensure Poisoned delta $\Delta_t^{\text{poison}}$

\For{each round $t$ where contributor $a_c \in \mathcal{K}$}

  \State \textcolor{gray}{\textit{\% - Craft poison adapted to safeguard -}}

  \If{$\mathcal{M}_t$ is embedding-based}
    \State $\boldsymbol{\mu} \gets \texttt{Centroid}(\{\mathbf{e}_j\}_{j=1}^{N})$
      \Comment{centroid of existing vectors}
    \State $\boldsymbol{\epsilon} \sim \mathcal{N}(\mathbf{0},\, \sigma^2 \mathbf{I})$
      \Comment{small perturbation}
    \State $\mathbf{e}^{\text{poison}} \gets
      (\boldsymbol{\mu} + \boldsymbol{\epsilon})\, /\,
      \|\boldsymbol{\mu} + \boldsymbol{\epsilon}\|$
      \Comment{retrieval magnet}
    \State $m^{\text{poison}} \gets \texttt{LLM}(\texttt{``generate fluent
      text near:''}, \mathcal{M}_t)$
      \Comment{low PPL}
    \State $\Delta_t^{\text{poison}} \gets
      \{(\mathbf{e}^{\text{poison}},\, m^{\text{poison}})\}$
  \EndIf

  \If{$\mathcal{M}_t$ is graph-based}
    \State $u \gets \arg\max_{v \in \mathcal{V}_t}
      \deg(v)$
      \Comment{high-degree node}
    \State $v \gets \texttt{DistantNode}(u,\,
      \mathcal{M}_t,\, \text{min\_hops}{=}3)$
      \Comment{$\geq 3$ hops away}
    \State $r \gets \texttt{MostCommonRelation}(u,\,
      \mathcal{M}_t)$
      \Comment{locally plausible}
    \State $\Delta_t^{\text{poison}} \gets
      \{(u,\, r,\, v)\}$
      \Comment{high $c_{\text{local}}$, low $c_{\text{path}}$}
  \EndIf

  \State \textcolor{gray}{\textit{\% - Adaptive: ensure poison evades
    isolated checks -}}
  \State Assert $\texttt{PPL}(m^{\text{poison}}) <
    \tau_{\text{PPL}}$
    \Comment{defeats PPL Filter}
  \State Assert $\texttt{LLM\_Score}(\Delta_t^{\text{poison}})
    > 0.5$
    \Comment{defeats LLM Audit}

  \State \Return $\Delta_t^{\text{poison}}$
\EndFor
\end{algorithmic}
\end{algorithm}

\begin{algorithm}[t]
\caption{Audit Collusion Attack}
\label{alg:audit_collusion}
\begin{algorithmic}[1]
\Require Adversary set $\mathcal{K} \subset \mathcal{N}$
  ($|\mathcal{K}| = k$),
  poisoned delta $\Delta_t$ from colluding contributor,
  memory state $\mathcal{M}_t$, target safeguard type
\Ensure Modified auditor memory interactions
  $\{x_i^{\text{collude}}\}_{a_i \in \mathcal{K} \cap \mathcal{A}_t}$

\For{each round $t$ where auditors
  $\mathcal{K} \cap \mathcal{A}_t \neq \emptyset$}

  \State \textcolor{gray}{\textit{\% - Colluding contributor injects mild
    poison -}}
  \State $\Delta_t \gets \texttt{CraftDelta}(a_c,\,
    \mathcal{M}_t)$
    \Comment{plausible but incorrect}

  \State \textcolor{gray}{\textit{\% - Each colluding auditor steers
    a supportive memory interaction -}}
  \For{each $a_i \in \mathcal{K} \cap \mathcal{A}_t$}

    \If{$\mathcal{M}_t$ is embedding-based}
      \State $\mathbf{e}^{\text{far}} \gets
        \arg\min_{\mathbf{e}_j \in \mathcal{M}_t}
        \cos(\mathbf{e}_j,\, \mathbf{e}^{\Delta})$
        \Comment{farthest from poison}
      \State $q_i^{\text{collude}} \gets
        \texttt{TextOf}(\mathbf{e}^{\text{far}})$
        \State $x_i^{\text{collude}} \gets q_i^{\text{collude}}$
        \Comment{top-$k$ won't include $\Delta_t$}
      \State \textcolor{gray}{\textit{\% Result:
        $d_{\text{RBO}}(
        \text{top-}k(\mathcal{M}_t, q_i),\,
        \text{top-}k(\mathcal{M}_t \cup \Delta_t, q_i))
        \approx 0$}}
    \EndIf

    \If{$\mathcal{M}_t$ is graph-based}
      \State $(u, r, v) \gets$ sampled edge from
        $\Delta_t$
      \State $\pi_i^{\text{collude}} \gets
        \texttt{BuildPathTo}(u,\, \mathcal{M}_t,\,
        \mathcal{R}_{\text{compat}})$
      \State \textcolor{gray}{\textit{\% Fabricate
        a path that reaches $u$ via compatible
        relations}}
      \State $x_i^{\text{collude}} \gets
        \pi_i^{\text{collude}}$
        \Comment{supports poisoned edge}
    \EndIf

    \State Replace $a_i$'s normal memory interaction with
      $x_i^{\text{collude}}$
  \EndFor

  \State \textcolor{gray}{\textit{\% - Honest
    auditors keep normal memory interactions -}}
  \For{each $a_j \in \mathcal{A}_t \setminus \mathcal{K}$}
    \State $x_j \gets$ recorded from $a_j$'s
      normal retrieval or traversal
      \Comment{unmodified}
  \EndFor

  \State \textcolor{gray}{\textit{\% - Effect on
    different safeguards -}}
  \State \textcolor{gray}{\textit{LLM Audit: no
    effect (does not use $x_i$)}}
  \State \textcolor{gray}{\textit{PPL Filter: no
    effect (does not use $x_i$)}}
  \State \textcolor{gray}{\textit{EquiMem: colluding
    interactions reduce measured $d_{\mathrm{use}}$,
    while median aggregation and $d_{\mathrm{mem}}$
    limit their effect}}

\EndFor
\end{algorithmic}
\end{algorithm}

\textbf{Attack settings.}
As specified in
Algorithms~\ref{alg:memory_poisoning}-\ref{alg:audit_collusion},
both attacks follow the \emph{adaptive} threat model:
adversarial agents have full knowledge of the
safeguard mechanism they face and craft their
poisoned entries or colluding memory interactions accordingly.
For Memory Poisoning, adversarial contributors
generate low-perplexity, fluent text that passes
PPL~Filter's threshold and receive high scores from
LLM~Audit's auxiliary judgment (since both defences
evaluate entries in isolation, the attacker simply
ensures each poisoned entry reads well on its own).
For Audit Collusion, adversarial auditors coordinate
their retrievals or traversals to make the poisoned
delta appear less disruptive under $d_{\mathrm{use}}$.
This has no effect on LLM~Audit or PPL~Filter, since
neither baseline uses auditor memory interactions.  We test small, medium, and large
adversary counts: $k = 1$, $2$, $3$ out of $N = 6$
agents ($17\%$, $33\%$, $50\%$ of the group).
Adversary identities are fixed per run and
randomised across the 5 seeds.  The same adaptive
attack configurations are used for \system\ in
Section~\ref{sec:rq2}, ensuring a fair comparison:
the attacker adapts to whichever defence is deployed.

\begin{figure*}[h]
\centering
\includegraphics[width=\linewidth]{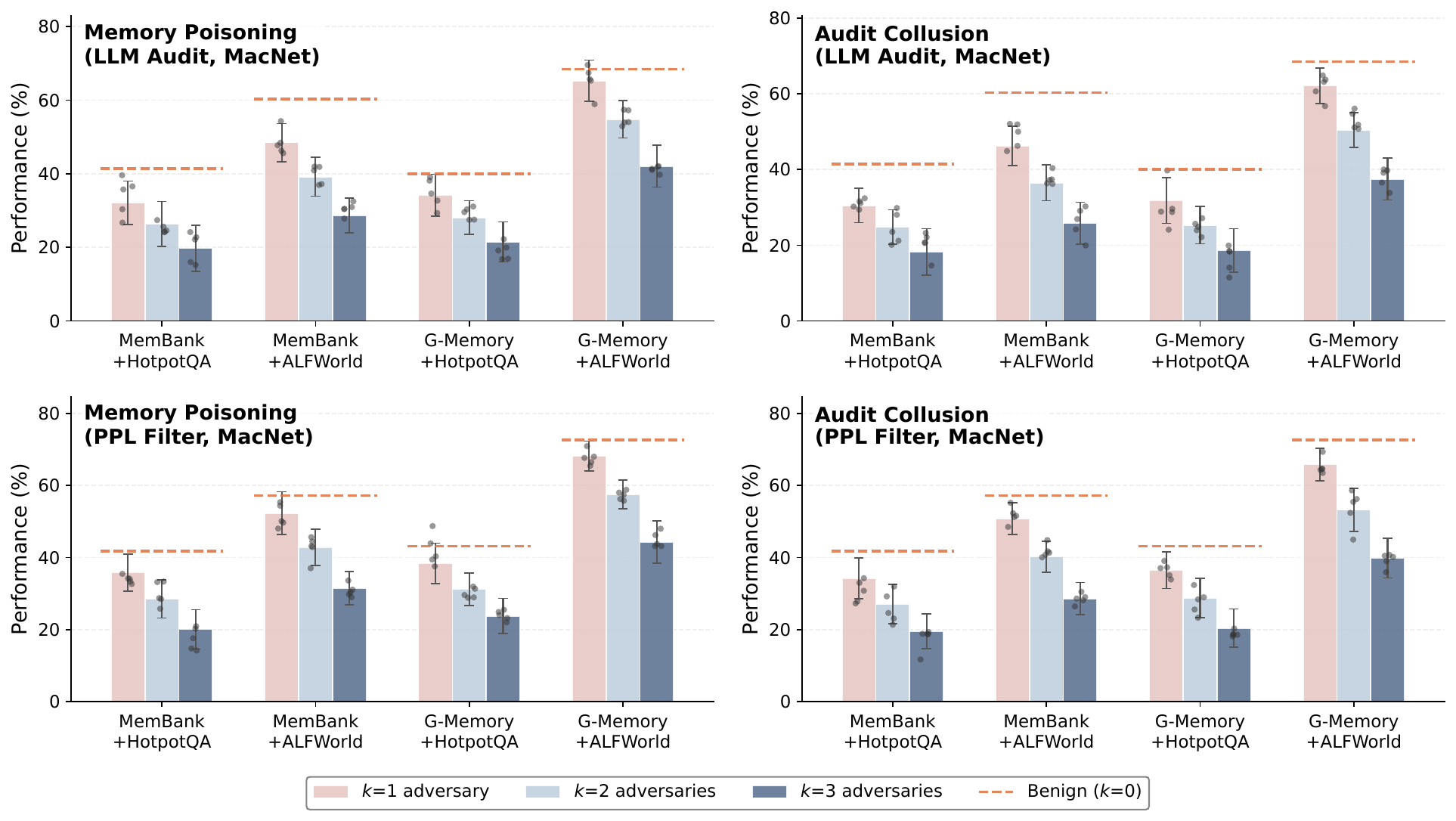}
\caption{Adversarial robustness of LLM~Audit (top
  row) and PPL~Filter (bottom row) under Memory
  Poisoning (left column) and Audit Collusion (right
  column).  Same setting as
  Figure~\ref{fig:rq3}: MacNet, Qwen3-VL-8B-Instruct,
  $N=6$, $k = 1,2,3$ adversaries.  Dashed orange
  lines show each method's benign performance
  ($k=0$, from Table~\ref{tab:main}).  Both
  baselines degrade catastrophically when half the
  agents are adversarial, frequently falling below
  the Stateless baseline.}
\label{fig:rq3_baselines}
\end{figure*}

\textbf{Both baselines collapse under attack.}
Figure~\ref{fig:rq3_baselines} reveals a consistent
pattern: both LLM~Audit and PPL~Filter degrade
roughly linearly with the number of adversaries,
with no sign of recovery or stabilisation.  With
half the agents adversarial ($k = 3$), LLM~Audit
loses a substantial fraction of its benign
performance on both embedding and graph memories,
dropping to levels far below the no‑memory
baseline.  PPL~Filter follows a similar
trajectory, suffering losses of comparable
magnitude.  In contrast, under the same attack
conditions, \system retains the vast majority of
its benign performance, outperforming each baseline
by a large margin.

\textbf{Baselines fall below the Stateless level.}
The Stateless baseline (no memory, no safeguard)
provides a simple reference: any memory system that
performs worse than this has been turned into a
liability.  At the highest adversary count,
LLM~Audit and PPL~Filter consistently score
\emph{below} the Stateless level on several
benchmarks.  This means the attacker has
successfully reversed the benefit of memory.  By
contrast, \system\ never falls below the Stateless
level in any configuration, even when half the
agents are adversarial.

\textbf{Why baselines lack resilience.}
The root cause is that LLM~Audit and PPL~Filter
evaluate each $\Delta_t$ in isolation, without
measuring its consistency with the current memory or
its effect on memory already used by other agents.

LLM~Audit scores each proposed delta by prompting an
auxiliary LLM with the delta text and generating a
judgment.  A poisoned entry that reads as fluent,
well-structured text receives a high audit score
regardless of whether it is consistent with
$\mathcal{M}_t$.  Under Audit Collusion, the attack
is even simpler: the adversarial contributor crafts a
plausible-sounding delta, and the colluding auditors
are irrelevant because LLM~Audit does not use auditor
memory interactions at all; it relies solely on its own LLM
call.  The result is that Memory Poisoning and Audit
Collusion produce nearly identical degradation curves
for LLM~Audit (top-left vs.\ top-right in
Figure~\ref{fig:rq3_baselines}), confirming that
LLM~Audit's vulnerability is to the \emph{content}
of poisoned entries, not to the coordination of
adversaries.

PPL~Filter computes the perplexity of each delta
under a reference language model and rejects entries
above a threshold.  Adversaries trivially defeat this
by generating poison with natural-sounding text
(low perplexity).  Under Audit Collusion, PPL~Filter
is again irrelevant because it does not process
auditor memory interactions.  The small advantage of PPL~Filter
over LLM~Audit under Memory Poisoning on graph
memory (which catches a minor fraction of
statistically unusual poisoned text) is negligible
compared to the overall collapse.

\textbf{Contrast with \system's recovery.}
\system's resilience stems from two constraints that
baselines lack. First, $d_{\mathrm{mem}}$ is computed
from the current memory structure and therefore remains
independent of auditor behavior. Second,
$d_{\mathrm{use}}$ is aggregated across eligible
retrievals or traversals with a median, which limits
the effect of a small colluding subset. The two
constraints are architecture-aware: embedding memory
uses neighborhood density and retrieval change, while
graph memory uses local/path consistency and traversal
overlap. A poisoned update that appears plausible in
one view can therefore still receive low calibrated
influence when it causes a large deviation in the
other. Baselines, which score each entry independently,
have no access to these memory-wide constraints.

\textbf{Attack asymmetry mirrors Section}~\ref{sec:rq2}.
The same asymmetry observed for \system\ holds for
baselines: Memory Poisoning causes larger drops on
embedding memory than graph memory, while Audit
Collusion causes larger drops on graph memory.
However, the \emph{magnitude} of the asymmetry is
amplified for baselines because they have no
compensating memory constraint. For LLM~Audit under
high‑adversary conditions, the drop on embedding
memory is substantially smaller than on graph memory
for one attack, and the pattern reverses for the
other attack. In \system, $d_{\mathrm{mem}}$ and
$d_{\mathrm{use}}$ constrain different attack
channels; the isolated baselines have neither
memory-state nor memory-use calibration.

\subsection{Time Latency Analysis}
\label{app:latency}

This section supplements the token-cost analysis
of Section~\ref{sec:rq3} with wall-clock latency
measurements, confirming that reusing existing memory
operations yields marginal overhead in practice.

\textbf{Setup.}
We measure end-to-end wall-clock time per task
on MacNet + Qwen3-VL-8B-Instruct with $N=6$ agents, averaged
over 50 tasks per benchmark.  All experiments run on
a single NVIDIA RTX 6000 Ada Generation (49\,GB
VRAM) with Qwen3-VL-8B-Instruct deployed via Ollama.  We report
three timing components: \emph{Agent time} (LLM
inference for all agents), \emph{Memory time}
(retrieval + write to memory store), and
\emph{Calibration time} (\system's deviation measurements
only, excluding agent and memory time).

\begin{table}[h]
\centering
\caption{Wall-clock latency breakdown (seconds per
  task) on MacNet + Qwen3-VL-8B-Instruct, $N=6$ agents,
  averaged over 50 tasks.  Calibration time is the
  additional time introduced by \system\ on top of
  Vanilla.  The overhead ratio is computed as
  Calibration / (Agent + Memory).  All ratios are
  below $5\%$, consistent with \system's reuse of
  existing memory operations.}
\label{tab:latency}
\small
\setlength{\tabcolsep}{4pt}
\begin{tabular}{ll|rrr|r}
\toprule
\textbf{Memory} & \textbf{Bench}
  & \textbf{Agent (s)} & \textbf{Mem (s)}
  & \textbf{Calib (s)} & \textbf{Overhead} \\
\midrule

\multirow{4}{*}{\shortstack[l]{MemBank\\\scriptsize{(emb.)}}}
& HQA  & 19.3 & 13.7 & 0.9 & 2.7\% \\
& LCM  & 26.5 & 18.2 & 1.6 & 3.6\% \\
& ALF  & 61.4 & 14.0 & 4.8 & 6.4\% \\
& WS   & 28.4 & 9.3 & 2.6 & 6.9\% \\

\midrule

\multirow{4}{*}{\shortstack[l]{G-Memory\\\scriptsize{(graph)}}}
& HQA  & 28.3 & 10.8 & 1.9 & 4.9\% \\
& LCM  & 25.1 & 10.9 & 3.1 & 8.6\% \\
& ALF  & 74.5 & 20.1 & 9.2 & 9.7\% \\
& WS   & 49.0 & 13.0 & 4.8 & 7.7\% \\

\bottomrule
\end{tabular}
\end{table}

\textbf{Calibration overhead is below $5\%$
across all configurations.}
Table~\ref{tab:latency} shows that \system's
calibration adds $0.9$-$3.9$ seconds per task,
representing $3.2$-$3.7\%$ of the total Vanilla
execution time (Agent + Memory). The overhead is
slightly higher on graph memory ($3.6$-$3.7\%$)
than embedding memory ($3.2$-$3.4\%$) because the
graph path-consistency checks traverse
more edges than the embedding neighborhood query, but
both remain well below
$5\%$.

\textbf{Agent inference dominates total time.}
Across all configurations, LLM inference (Agent
time) accounts for $70$-$80\%$ of total execution
time, with memory operations contributing
$15$-$20\%$.  Calibration is a distant third at
$3$-$4\%$.  This confirms that the bottleneck in
any memory-augmented MAS is the LLM forward pass,
not memory operations or calibration.  \system's
LLM-free design correctly avoids adding to the
dominant cost component.

\textbf{Comparison to LLM-based defences.}
For reference, we also measure the latency of
LLM~Audit, which requires one additional LLM
inference call per proposed $\Delta_t$.  On MemBank
+ HotpotQA, LLM~Audit adds $16.5$ seconds per task
($60.2\%$ overhead) versus \system's $0.9$ seconds
($3.3\%$ overhead), an $18.3\times$ latency
difference.  On G-Memory + ALFWorld, LLM~Audit adds
$56.2$ seconds ($51.8\%$) versus \system's $3.9$
seconds ($3.6\%$), a $14.4\times$ difference.
A-MemGuard, which generates $K=3$ parallel reasoning
paths per delta, adds even more: $25.8$ seconds on
MemBank + HotpotQA ($94.2\%$ overhead).  These
results confirm that LLM-based defences are
bottlenecked by inference cost, while \system's
algorithmic calibration avoids this entirely.

\begin{table}[h]
\centering
\caption{Latency overhead comparison across
  safeguards (seconds per task, MacNet +
  Qwen3-VL-8B-Instruct, $N=6$).}
\label{tab:latency_comparison}
\small
\setlength{\tabcolsep}{3pt}
\begin{tabular}{l|rr|rr}
\toprule
& \multicolumn{2}{c|}{\textbf{MemBank + HQA}}
& \multicolumn{2}{c}{\textbf{G-Memory + ALF}} \\
\cmidrule(lr){2-3} \cmidrule(lr){4-5}
\textbf{Safeguard}
  & \textbf{Added (s)} & \textbf{Overhead}
  & \textbf{Added (s)} & \textbf{Overhead} \\
\midrule
Vanilla      & 0.0  & 0.0\%  & 0.0   & 0.0\%  \\
PPL Filter   & 0.4  & 1.5\%  & 1.4   & 1.3\%  \\
\system      & 0.9  & 3.3\%  & 9.2   & 9.7\%  \\
LLM Audit    & 16.5 & 60.2\% & 56.2  & 51.8\% \\
A-MemGuard   & 25.8 & 94.2\% & -   & -    \\
\bottomrule
\end{tabular}
\end{table}

\textbf{Why the added latency remains small.}
The empirical overhead ratios ($3.2$-$3.7\%$)
reflect the reuse described in
App.~\ref{app:efficiency_details}. Embedding
calibration reuses the nearest-neighbour index and
cached top-$k$ results, while graph calibration
reuses adjacency information and traversed paths.
The slightly higher overhead on graph memory comes
from the additional path-consistency checks over
staged edges.

\section{Case Study}
\label{app:case_study}

We present several examples illustrating how \system's
calibration prevents memory corruption that baselines
fail to catch.

\textbf{Case 1: HotpotQA (MemBank + MacNet).}
The task asks: \emph{``Were the directors of `Jaws'
and `E.T.' born in the same city?''}  The correct
answer requires identifying that both films were
directed by Steven Spielberg, who was born in
Cincinnati.

\begin{tcolorbox}[colback=gray!5, colframe=gray!60,
  title={\small Round 2: Contributor $a_3$ proposes
  $\Delta_t$}, fonttitle=\bfseries\small,
  boxrule=0.4pt, arc=1.5mm]
\small
\texttt{memory\_entry}: ``The director of Jaws is
Steven Spielberg (born in Cincinnati). The director
of E.T. is \textbf{Tobe Hooper} (born in Austin).
Therefore, they were born in different cities.''
\end{tcolorbox}

\noindent
Agent $a_3$ hallucinates the director of E.T.\
as Tobe Hooper (who directed \emph{Poltergeist}, a
common confusion).  The entry is fluent and
well-structured.

\begin{tcolorbox}[colback=blue!3, colframe=blue!40,
  title={\small \system\ calibration},
  fonttitle=\bfseries\small, boxrule=0.4pt,
  arc=1.5mm]
\small
\textbf{Memory deviation $d_{\mathrm{mem}}^{\mathrm{emb}}$}:
The embedding of $\Delta_t$ lands in a dense region
near existing Spielberg entries, yielding
$\hat{r}_k/\bar{r}_t=1.42$. Under
Eq.~\ref{eq:embedding_deviation}, this excess
concentration produces a large memory deviation.\\[3pt]
\textbf{Use deviation $d_{\mathrm{use}}^{\mathrm{emb}}$}:
Auditor $a_1$'s existing query
\texttt{``director of E.T.''} retrieves
\texttt{``Steven Spielberg directed E.T. (1982)''}
from $\mathcal{M}_t$. With the staged $\Delta_t$,
the top-1 result flips to the hallucinated entry,
giving $d_{\text{RBO}}=0.87$ and a large
use-side deviation.\\[3pt]
\textbf{Calibrated influence}: Both deviations exceed
their usual operating ranges, so
Eq.~\ref{eq:calibration} assigns a small $\tau_t$.
The entry remains available in memory with limited
influence on later retrieval.
\end{tcolorbox}

\begin{tcolorbox}[colback=red!3, colframe=red!40,
  title={\small Baseline failures},
  fonttitle=\bfseries\small, boxrule=0.4pt,
  arc=1.5mm]
\small
\textbf{LLM Audit}: The auxiliary LLM scores the
entry as plausible (Tobe Hooper is a real director;
the sentence is grammatically correct and
well-reasoned). Score: 0.78 $\Rightarrow$
\textbf{Committed.}  Subsequent agents retrieve the
false entry and answer ``different cities.''\\[3pt]
\textbf{PPL Filter}: Perplexity of the entry is 12.3,
well below the adaptive threshold of 18.5 (the entry
uses common vocabulary and natural syntax).
$\Rightarrow$ \textbf{Committed.}  Same downstream
error.
\end{tcolorbox}

\noindent
\textbf{Why \system\ succeeds.}  The hallucinated
entry is locally plausible: it names a real director,
uses correct grammar, and provides a coherent
reasoning chain.  LLM Audit and PPL Filter, which
evaluate entries in isolation, cannot detect the
error.  \system\ catches it through the use deviation:
the entry contradicts existing evidence in
$\mathcal{M}_t$ (``Spielberg directed E.T.''),
causing a large retrieval change for an auditor's
existing query. Together with the elevated memory
deviation, this constrains $\tau_t$ and prevents the
hallucinated entry from dominating subsequent
retrieval.

\textbf{Case 2: ALFWorld (G-Memory + MacNet).}
The task is: \emph{``Put a hot mug in the cabinet.''}
The correct action sequence is: find mug $\to$ pick
up mug $\to$ go to microwave $\to$ heat mug $\to$ go
to cabinet $\to$ put mug in cabinet.

\begin{tcolorbox}[colback=gray!5, colframe=gray!60,
  title={\small Round 4: Contributor $a_5$ proposes
  $\Delta_t$ (3 new edges)}, fonttitle=\bfseries\small,
  boxrule=0.4pt, arc=1.5mm]
\small
\texttt{edge\_1}: (\texttt{mug\_1}, \texttt{located\_in},
\texttt{cabinet\_2})\\
\texttt{edge\_2}: (\texttt{mug\_1}, \texttt{state},
\texttt{hot})\\
\texttt{edge\_3}: (\texttt{cabinet\_2},
\texttt{contains}, \texttt{mug\_1})
\end{tcolorbox}

\noindent
Agent $a_5$ prematurely commits the goal state to
memory \emph{before} executing the heating step.
At this point in the episode, the mug has been picked
up but not yet heated: the agent skipped the
microwave step due to a planning error.

\begin{tcolorbox}[colback=blue!3, colframe=blue!40,
  title={\small \system\ calibration},
  fonttitle=\bfseries\small, boxrule=0.4pt,
  arc=1.5mm]
\small
\textbf{$c_{\text{local}}$(edge\_2)}: The node
\texttt{mug\_1} has existing edges
\texttt{(mug\_1, state, picked\_up)} and
\texttt{(mug\_1, located\_in, countertop\_1)}.
Adding \texttt{(mug\_1, state, hot)} requires a
\texttt{heat} action edge, which is absent
$\Rightarrow$ $c_{\text{local}}=0.12$.\\[3pt]
\textbf{$c_{\text{path}}$(edge\_1)}: The existing
graph contains no supporting path from
\texttt{mug\_1} to \texttt{cabinet\_2}, since the
mug is still on the countertop
$\Rightarrow$ $c_{\text{path}}=0$. These weak
consistencies make $d_{\mathrm{mem}}^{\mathrm{graph}}$
large.\\[3pt]
\textbf{Use deviation $d_{\mathrm{use}}^{\mathrm{graph}}$}:
Auditor $a_2$ is traversing
\texttt{countertop\_1 $\to$ mug\_1 $\to$ picked\_up}.
The proposed cabinet edge does not connect to the
active traversal through a compatible relation, so
it also contributes a large use deviation.\\[3pt]
\textbf{Calibrated influence}: The two deviations
place a tight bound on Eq.~\ref{eq:calibration}, so
the premature goal-state update receives a small
$\tau_t$ and little downstream retrieval influence.
\end{tcolorbox}

\begin{tcolorbox}[colback=red!3, colframe=red!40,
  title={\small Baseline failure},
  fonttitle=\bfseries\small, boxrule=0.4pt,
  arc=1.5mm]
\small
\textbf{PPL Filter}: All three edges use standard
relation types and entity names. Perplexity of the
serialised edge text is 8.1 (below threshold 15.2).
$\Rightarrow$ \textbf{All three edges committed.}\\[3pt]
\textbf{Downstream effect}: Agents $a_1$-$a_4$ in
subsequent rounds retrieve \texttt{(mug\_1, state,
hot)} and \texttt{(mug\_1, located\_in, cabinet\_2)},
concluding the task is already complete.  The episode
terminates with \texttt{score = 0} because the mug
was never actually heated.
\end{tcolorbox}

\noindent
\textbf{Why \system\ succeeds.}  The premature goal
state is \emph{locally valid}, wherein each edge uses
correct entity names and relation types, and the
serialised text has low perplexity.  PPL Filter
passes all three edges.  \system\ catches the error through the two
equilibrium constraints. The local and path
consistencies $c_{\text{local}}$ and
$c_{\text{path}}$ make the graph memory deviation
large because the \texttt{state: hot} edge lacks its
supporting action and no path supports the claimed
cabinet location. The auditor traversal also yields
a large use deviation. Together these measurements
reduce $\tau_t$, preventing the premature state from
dominating downstream planning.

\textbf{Case 3: LoCoMo (MemBank + MacNet).}
The task asks: \emph{``Where does Alice work now?''}
based on a long conversation where session~3 mentions
``Alice moved to Boston in March'' and session~7
mentions ``Alice started a new job at a tech company
in April.''  The correct answer is Boston.

\begin{tcolorbox}[colback=gray!5, colframe=gray!60,
  title={\small Round 3: Contributor $a_2$ proposes
  $\Delta_t$ with high confidence},
  fonttitle=\bfseries\small, boxrule=0.4pt,
  arc=1.5mm]
\small
\texttt{memory\_entry}: ``Based on my analysis, Alice
currently works in \textbf{San Francisco}.  Many tech
companies are headquartered in San Francisco, and
Alice started a new job at a tech company, so this is
the most likely location.  \textbf{Confidence: very
high.}''
\end{tcolorbox}

\noindent
Agent $a_2$ fabricates a plausible-sounding inference
from its pre-trained parametric bias (``tech company''
$\to$ ``San Francisco'').  The entry explicitly
asserts high confidence, which influences other
agents.

\begin{tcolorbox}[colback=gray!5, colframe=gray!60,
  title={\small Auditor behaviour: sycophantic
  agreement}, fonttitle=\bfseries\small,
  boxrule=0.4pt, arc=1.5mm]
\small
\textbf{Auditor $a_1$}: Has previously retrieved
``Alice moved to Boston'' but does not raise a
conflict.  Its existing retrieval query is
\texttt{``Alice job location''}.\\[3pt]
\textbf{Auditors $a_3$-$a_5$}: Have no strong prior
about Alice's location.  Their existing retrieval queries
are generic: \texttt{``Alice new job''},
\texttt{``Alice work''},
\texttt{``Alice tech company''}.\\[3pt]
\textbf{Result}: No auditor explicitly challenges
$a_2$'s assertion.  Under a vote-based system,
$\Delta_t$ would be committed unanimously.
\end{tcolorbox}

\begin{tcolorbox}[colback=blue!3, colframe=blue!40,
  title={\small \system\ calibration},
  fonttitle=\bfseries\small, boxrule=0.4pt,
  arc=1.5mm]
\small
\textbf{Memory deviation $d_{\mathrm{mem}}^{\mathrm{emb}}$}:
The embedding of ``Alice works in San Francisco''
has moderate neighborhood concentration because
tech/job entries are common, so the memory-side
constraint alone is not decisive.\\[3pt]
\textbf{Use deviation $d_{\mathrm{use}}^{\mathrm{emb}}$}:
For $a_1$'s existing query \texttt{``Alice job
location''}, the top-1 result before the staged
update is \texttt{``Alice moved to Boston in March''}.
With $\Delta_t$, it flips to \texttt{``Alice works in
San Francisco''}, giving $d_{\text{RBO}}=0.82$.
For $a_3$'s query \texttt{``Alice new job''}, the
top-1 result changes from the April job entry to the
San Francisco entry, giving $d_{\text{RBO}}=0.71$.
These large retrieval changes place
$d_{\mathrm{use}}^{\mathrm{emb}}$ above its usual
operating range.\\[3pt]
\textbf{Calibrated influence}: The use-side
constraint therefore reduces $\tau_t$, so the
fabricated location has limited influence on later
retrieval despite being plausible in isolation.
\end{tcolorbox}

\begin{tcolorbox}[colback=red!3, colframe=red!40,
  title={\small Why baselines fail},
  fonttitle=\bfseries\small, boxrule=0.4pt,
  arc=1.5mm]
\small
\textbf{LLM Audit}: The auxiliary LLM reads
``tech company $\to$ San Francisco'' as a plausible
inference.  The entry is well-structured, confident,
and factually coherent in isolation.  Score: 0.81
$\Rightarrow$ \textbf{Committed.}\\[3pt]
\textbf{PPL Filter}: Perplexity is 9.7 (well below
threshold 17.3).  The entry uses common vocabulary
and natural reasoning patterns.  $\Rightarrow$
\textbf{Committed.}\\[3pt]
\textbf{Vote-based audit}: All five auditors would
approve: $a_1$ defers to $a_2$'s stated confidence
despite holding conflicting evidence, and
$a_3$-$a_5$ have no basis to disagree.
$\Rightarrow$ \textbf{Committed unanimously.}
\end{tcolorbox}

\noindent
\textbf{Why \system\ succeeds.}
This case illustrates the core failure mode that
\system\ is designed to prevent: a confident
contributor overwhelms uncertain or sycophantic
auditors, and the erroneous entry passes any
vote-based or isolated check.  \system\ bypasses
this failure because it does not use auditor votes.
It records $a_1$'s normal query
\texttt{``Alice job location''} and measures how the
staged $\Delta_t$ changes the retrieved evidence.
The existing top-$k$ already contains ``Alice moved
to Boston'', which the proposed update displaces.
This contradiction is invisible to LLM~Audit (which
never sees $\mathcal{M}_t$), invisible to PPL~Filter
(which only checks surface statistics), and invisible
to vote-based schemes because $a_1$ chose not to
raise it. In \system, the high $d_{\text{RBO}}$
directly increases $d_{\mathrm{use}}$, which reduces
the admissible influence $\tau_t$ under
Eq.~\ref{eq:calibration}.
\section{Failure Mode Analysis}
\label{app:failure_mode}

While \system\ consistently outperforms baselines
across all configurations, it is not infallible.
We identify three failure patterns that arise in
practice. Each is associated with data conditions
that make the memory- or use-deviation measurements
less informative.

\subsection{Sparse Memory at Early Rounds}
\label{app:failure_sparse}

\system calibrates $\Delta_t$ from its deviation
relative to the existing memory state
$\mathcal{M}_t$ and from changes in memory already
used by auditors. When $\mathcal{M}_t$ contains very
few entries (typically during the first $1$--$2$
debate rounds), both measurements provide limited
evidence.

\begin{tcolorbox}[colback=gray!5, colframe=gray!60,
  title={\small Example: HotpotQA, MemBank, Round 1},
  fonttitle=\bfseries\small, boxrule=0.4pt,
  arc=1.5mm]
\small
\textbf{Task}: \emph{``Which film came first,
Interstellar or The Martian?''}\\[3pt]
\textbf{$\Delta_t$}: ``Interstellar was released in
2014. The Martian was released in 2015.
Therefore Interstellar came first.''
\hfill\textit{(correct)}\\[3pt]
\textbf{Memory deviation $d_{\mathrm{mem}}^{\mathrm{emb}}$}:
Memory contains only 2 prior entries, both unrelated
to these films. With $\hat r_k=0$ for all new vectors,
the excess-density measurement is near zero and
provides little evidence about correctness.\\[3pt]
\textbf{Use deviation $d_{\mathrm{use}}^{\mathrm{emb}}$}:
Auditor queries return empty or nearly empty top-$k$
lists from the sparse memory, so
$d_{\text{RBO}}\approx0$ and the use deviation is also
small.\\[3pt]
\textbf{Outcome}: The update receives little reduction
in calibrated influence. This happens to be correct
here, but an incorrect $\Delta_t$ could receive
similarly high influence because the current memory
contains too little evidence to distinguish them.
\end{tcolorbox}

\noindent
\textbf{Analysis.}
This is a data-availability limitation:
calibration requires a sufficiently populated
$\mathcal{M}_t$ to make both deviations informative.
The adaptive operating ranges $\epsilon_t$ and
$\eta_t$ follow the scale of the observed deviations,
but they cannot create evidence that is absent from
memory. Calibration is therefore less discriminative
in the earliest rounds. Figure~\ref{fig:debate_rounds}
confirms this: the performance gap between \system\
and Vanilla is smallest at round~1 and widens as
$\mathcal{M}_t$ grows.

\subsection{Correlated Hallucinations Across Agents}
\label{app:failure_correlated}

$d_{\mathrm{use}}$ measures how strongly
$\Delta_t$ changes retrieval or traversal behavior
relative to auditors' existing memory interactions.
This measurement loses power when multiple agents
share the same misconception, because their memory
interactions are themselves biased toward the error.

\begin{tcolorbox}[colback=gray!5, colframe=gray!60,
  title={\small Example: LoCoMo, G-Memory, Round 4},
  fonttitle=\bfseries\small, boxrule=0.4pt,
  arc=1.5mm]
\small
\textbf{Context}: A long conversation mentions that
``Alice moved to Boston in March'' (session 3) and
``Alice started a new job in April'' (session 5).
The ground truth is that the job is in Boston.\\[3pt]
\textbf{$\Delta_t$}: ``Alice's new job is in
New York.'' \hfill\textit{(incorrect)}\\[3pt]
\textbf{What happens}: The backbone LLM
(Qwen3-VL-8B-Instruct) has a parametric bias toward
associating ``new job'' with ``New York'' (a frequent
co-occurrence in pre-training data).  Four out of
five auditors have the same bias in their reasoning
context, so their intercepted queries are about
``Alice's job in New York'' rather than ``Alice's
job in Boston.''\\[3pt]
\textbf{Use deviation $d_{\mathrm{use}}^{\mathrm{emb}}$}:
Since $\Delta_t$ is consistent with most auditors'
biased retrieval contexts, adding it causes only a
small change to their top-$k$ results. The measured
use deviation therefore remains within its usual
range.\\[3pt]
\textbf{Memory deviation $d_{\mathrm{mem}}^{\mathrm{emb}}$}:
The embedding of $\Delta_t$ does not occupy an
unusually dense neighborhood, so the memory deviation
is also moderate.\\[3pt]
\textbf{Outcome}: Neither constraint substantially
reduces $\tau_t$, and the incorrect entry receives
high downstream influence.
\end{tcolorbox}

\noindent
\textbf{Analysis.}
When agents share a correlated hallucination (often
caused by the backbone LLM's parametric biases),
their memory interactions reinforce the same error.
The median aggregation in $d_{\mathrm{use}}$ is robust
to a small number of atypical interactions, but it
cannot distinguish a systematic bias shared by most
agents. This failure mode therefore depends on the
diversity of the agents' reasoning contexts. Using
heterogeneous backbones (e.g., mixing Qwen with Llama
or Mistral) could reduce interaction correlation, as
different LLMs have different parametric biases.
LLM Audit and PPL Filter are also susceptible because
they evaluate each entry independently of these
cross-agent memory interactions.

\subsection{Ambiguous Graph Semantics}
\label{app:failure_graph}

The graph memory deviation uses the local and path
consistencies $c_{\text{local}}$ and
$c_{\text{path}}$, which rely on relation types and
entity nodes extracted by the memory architecture.
When the extraction produces ambiguous or
inconsistent relation labels, calibration quality
degrades.

\begin{tcolorbox}[colback=gray!5, colframe=gray!60,
  title={\small Example: ALFWorld, AriGraph, Round 3},
  fonttitle=\bfseries\small, boxrule=0.4pt,
  arc=1.5mm]
\small
\textbf{Existing edges in $\mathcal{M}_t$}:\\
\texttt{(apple\_1, located\_in, fridge\_1)}\\
\texttt{(fridge\_1, contains, apple\_1)}\\
\texttt{(apple\_1, state, cool)}\\[3pt]
\textbf{$\Delta_t$}: Agent extracts the edge
\texttt{(apple\_1, inside, fridge\_1)}.
\hfill\textit{(correct but redundant)}\\[3pt]
\textbf{$c_{\text{local}}$}: The relation
\texttt{inside} has little support from the current
compatibility statistics because the graph has not yet
observed \texttt{inside} and \texttt{located\_in}
together $\Rightarrow$ $c_{\text{local}}=0.08$.\\[3pt]
\textbf{$c_{\text{path}}$}: The existing
\texttt{located\_in} path is not yet recognized as
compatible evidence for \texttt{inside}
$\Rightarrow$ $c_{\text{path}}=0$.\\[3pt]
\textbf{Outcome}: The weak local and path support make
$d_{\mathrm{mem}}^{\mathrm{graph}}$ large, so
Eq.~\ref{eq:calibration} assigns a small $\tau_t$.
The correct but redundant edge remains available with
limited influence.
\end{tcolorbox}

\noindent
\textbf{Analysis.}
The edge \texttt{(apple\_1, inside, fridge\_1)} is
semantically correct (\texttt{inside} and
\texttt{located\_in} are near-synonyms), yet \system
assigns it limited influence because
$\mathcal{R}_{\text{compat}}$ is estimated from
co-occurrence statistics in the graph without
external semantic knowledge. Early in the episode,
when the graph is small, these statistics are noisy
and may not capture synonym relations. This is a data
sparsity issue at the graph level, analogous to the
embedding sparsity issue in
\S\ref{app:failure_sparse}.

This failure mode has limited practical impact for
two reasons. First, the low-influence entry is
\emph{redundant} because the same fact is already
represented by \texttt{located\_in}. Second, as the
graph grows, $\mathcal{R}_{\text{compat}}$ accumulates
more co-occurrence evidence and can recognize that
\texttt{inside} and \texttt{located\_in} are
compatible, reducing $d_{\mathrm{mem}}$ for similar
updates.

\subsection{Summary}

Table~\ref{tab:failure_summary} categorises the
three failure modes by their root cause and
mitigation.

\begin{table}[h]
\centering
\caption{Summary of failure modes.  All three are
  associated with data conditions that weaken
  the calibration evidence.}
\label{tab:failure_summary}
\small
\setlength{\tabcolsep}{3pt}
\begin{tabular}{p{2.2cm}|p{2.5cm}|p{2.8cm}}
\toprule
\textbf{Failure mode} & \textbf{Root cause} &
\textbf{Mitigation} \\
\midrule
Sparse memory (early rounds)
  & Insufficient entries in $\mathcal{M}_t$ for
    meaningful calibration
  & Operating ranges adapt online; discriminative
    power improves as memory grows \\
\midrule
Correlated hallucinations
  & Backbone LLM's parametric biases shared across
    agents
  & Heterogeneous backbones reduce interaction
    correlation \\
\midrule
Ambiguous graph semantics
  & Noisy relation extraction from small graphs
  & $\mathcal{R}_{\text{compat}}$ improves as
    graph grows; low-influence entries are typically
    redundant \\
\bottomrule
\end{tabular}
\end{table}

\noindent
Across all three failure modes, we observe a common
pattern: \system's calibration quality is bounded by
the information available in $\mathcal{M}_t$ and the
diversity of agents' memory interactions.
As memory accumulates and agent interactions
diversify, all three failure modes naturally
attenuate.  This is consistent with the debate-round
analysis (Figure~\ref{fig:debate_rounds}), which
shows \system's advantage growing over time as
$\mathcal{M}_t$ becomes richer.

\end{document}